\documentclass{new_tlp}
\usepackage{mathptmx}
\DeclareMathAlphabet{\mathcal}{OMS}{cmsy}{m}{n}

\usepackage{graphicx}
\usepackage{latexsym}

\usepackage{comment}
\usepackage{hyperref}

\usepackage{enumitem}
\setlist[itemize]{nolistsep}
\setlist[description]{nolistsep}
\setlist[enumerate]{nolistsep}
\usepackage{amsmath}
\usepackage{amssymb}

\newtheorem{example}{Example}
\newtheorem{definition}{Definition}
\newtheorem{proposition}{Proposition}
\newtheorem{lemma}{Lemma}
\newtheorem{theorem}{Theorem}

\usepackage{ifthen}
\newcommand{\brifnotempty}[1]{\ifthenelse{\equal{#1}{}}{}{ \br{#1}}}
\newenvironment{lemma*}[2][]
	{\pagebreak[2] \par \noindent \textbf{Lemma~\ref{#2}\brifnotempty{#1}.}\it}{\par}
\newenvironment{theorem*}[2][]
	{\pagebreak[2] \par \noindent \textbf{Theorem~\ref{#2}\brifnotempty{#1}.}\it}{\par}
\newenvironment{proposition*}[2][]
	{\pagebreak[2] \par \noindent \textbf{Proposition~\ref{#2}\brifnotempty{#1}.}\it}{\par}
\newenvironment{corollary*}[2][]
	{\pagebreak[2] \par \noindent \textbf{Corollary~\ref{#2}\brifnotempty{#1}.}\it}{\par}

\newcommand{\sign}{\Sigma}

\newcommand{\HT}{\mathcal{HT}}
\newcommand{\as}[1]{\ensuremath{\mathcal{AS}(#1)}} 
\newcommand{\classP}{\mathcal{C}} 

\newcommand{\la}{\leftarrow}
\newcommand{\nf}{not\,}

\newcommand{\head}[1]{\ensuremath{\mathit{H}(#1)}}
\newcommand{\body}[1]{\ensuremath{\mathit{B}(#1)}}
\newcommand{\pbody}[1]{\ensuremath{\mathit{B}^+(#1)}}
\newcommand{\nbody}[1]{\ensuremath{\mathit{B}^-(#1)}}
\newcommand{\nnbody}[1]{\ensuremath{\mathit{B}^{--}(#1)}}

\newcommand{\rhead}[1]{\ensuremath{\mathit{head}(#1)}}
\newcommand{\rbody}[1]{\ensuremath{\mathit{body}(#1)}}
\newcommand{\bodywoq}[1]{\ensuremath{\mathit{B}^{\setminus{q}}(#1)}}
\newcommand{\headwoq}[1]{\ensuremath{\mathit{H}^{\setminus{q}}(#1)}}
\newcommand{\redBody}[2]{\ensuremath{\mathit{B}^{\text{\textbackslash}#2}(#1)}}
\newcommand{\redHead}[2]{\ensuremath{\mathit{H}^{\text{\textbackslash}#2}(#1)}}

\newcommand{\setm}{\text{\textbackslash}} 

\newcommand{\asdual}[2]{\ensuremath{\mathcal{D}^{#1}_{as}(#2)}}
\newcommand{\NF}[1]{\ensuremath{NF(#1)}}

\newcommand{\fSP}[2]{\ensuremath{\mathsf{f}_{SP}(#1,#2)}}
\newcommand{\fSPOP}{\ensuremath{\mathsf{f}_{SP}}}
\newcommand{\FSPOP}{\ensuremath{F_{SP}}}
\newcommand{\fgt}{\ensuremath{\mathsf{f}}}
\newcommand{\f}[2]{\ensuremath{\mathsf{f}(#1,#2)}}
\newcommand{\classF}{\ensuremath{\mathsf{F}}}
\newcommand{\fSem}{\ensuremath{\mathsf{f}_{Sem}}}
\newcommand{\classMFSP}{\ensuremath{M_{\tuple{P,V}}}}
\newcommand{\forget}{\ensuremath{\mathsf{forget}}}

\newcommand{\cmr}[2]{\ensuremath{\mathit{r}_{#1}^{#2}}}

\newcommand{\tuple}[1]{\ensuremath{\langle#1\rangle}}
\newcommand{\SRel}{\ensuremath{\mathcal{R}_{\tuple{P,V}}}}
\newcommand{\Rel}{\ensuremath{Rel}}
\newcommand{\RA}{\ensuremath{R}}
\newcommand{\spF}{\sf SP}

\newcommand{\strong}{strong}
\newcommand{\weak}{weak}
\newcommand{\sem}{sem}
\newcommand{\Min}[1]{\ensuremath{\mathcal{MIN}(#1)}}
\newcommand{\Sas}{Sas}

\newcommand{\pSE}{{\bf (SE)}}

\newcommand{\pSP}{{\bf (SP)}}

\newcommand{\pSI}{{\bf (SI)}}

\newcommand{\pwSP}{{\bf (wSP)}}

\usepackage{url}

\begin{document}

\title{A Syntactic Operator for Forgetting that Satisfies Strong Persistence}

\author[Matti Berthold et al.]
{MATTI BERTHOLD\\
Technische Universit\"at Dresden and Universit\"at Leipzig, Germany
\and RICARDO GON\c CALVES, MATTHIAS KNORR, JO\~AO LEITE\\
         NOVA LINCS, Departamento de Inform\'atica, Universidade Nova de Lisboa, Portugal}
\maketitle

\begin{abstract}
Whereas the operation of \emph{forgetting} has recently seen a considerable amount of attention in the context of Answer Set Programming (ASP), most of it has focused on theoretical aspects, leaving the practical issues largely untouched. Recent studies include results about what sets of properties operators should satisfy, as well as the abstract characterization of several operators and their theoretical limits. However, no concrete operators have been investigated.

In this paper, we address this issue by presenting the first concrete operator that satisfies \emph{strong persistence} -- a property that seems to best capture the essence of forgetting in the context of ASP -- whenever this is possible, and many other important properties. The operator is syntactic, limiting the computation of the forgetting result to manipulating the rules in which the atoms to be forgotten occur, naturally yielding a forgetting result that is close to the original program. 
\end{abstract}

\section{Introduction}
\label{sec:intro}

Unlike belief change operations such as \emph{revision}, \emph{update} and \emph{contraction}, which have deserved ample attention for over three decades now, 
only recently has the Knowledge Representation and Reasoning research community recognized the operation of \emph{forgetting} or \emph{variable elimination} -- i.e., removing from a knowledge base information that is no longer needed -- as a critical operation, as witnessed by the amount of work developed for different logical formalisms (cf. the survey \cite{EiterK2018}). 

The operation of \emph{forgetting} is most useful when we wish to eliminate (temporary) variables introduced to represent auxiliary concepts, with the goal of restoring the declarative nature of some knowledge base, or just to simplify it. Furthermore, it is becoming increasingly necessary to properly deal with legal and privacy issues, including, for example, to enforce the new EU General Data Protection Regulation which includes the \emph{Right to be Forgotten}. The operation of \emph{forgetting} has been applied in cognitive robotics \cite{LinR97,LiuW11,RajaratnamLPT14}, resolving conflicts \cite{LLM03,ZhangF06,EiterW08,LangM10}, ontology abstraction and comparison \cite{WWTP10,KontchakovWZ10}, among others \cite{ZZ-AI09,AlferesKW13}, which further witnesses its importance.

This has also triggered an increasing interest in the development of forgetting operators for Logic Programs (LP) in particular Answer Set Programming (ASP)~\cite{GL91}, see, e.g., \cite{ZhangF06,EiterW08,WangWZ13,KnorrA14,WangZZZ14,DelgrandeW15,GoncalvesKL-ECAI16,GKL17,GJKLW-AAAI19} or \cite{GoncalvesKL16,Leite17} for a recent survey.

Despite the significant amount of interest in \emph{forgetting} in the context of ASP, most research has focused on theoretical aspects, such as semantically characterizing various operators and investigating their properties. We now have a rather good understanding of the landscape of properties and operators, and their theoretical limits. Notably, it is now well established that \emph{strong persistence} \cite{KnorrA14} -- a property inspired by \emph{strong equivalence}, which requires that there be a correspondence between the answer
sets of a program before and after forgetting a set of atoms, and
that such correspondence be preserved in the presence of additional
rules not containing the atoms to be forgotten  -- best captures the essence of forgetting in the context of ASP, even if it is not always possible to forget some atom from a program while obeying such property \cite{GoncalvesKL-ECAI16}. Also, there exists a semantic characterisation of the class of forgetting operators that satisfy \emph{strong persistence} whenever that is possible.

Most of these theoretical results, however, have not been accompanied by investigations into the practical aspects of these operators. The characterization of the operators found in the literature is usually provided through the set of HT-models of the result, usually appealing to some general method to generate a concrete program from those models, such as the method that relies on the notion of counter-models in HT \cite{CabalarF07} or one of its variations \cite{WangWZ13,WangZZZ14}. Whereas relying on such methods to produce a concrete program is important in the sense of being a proof that such a program exists, it suffers from several severe drawbacks when used in practice:

\begin{itemize}
\item In general, it produces programs with a very large number of rules. For example, if we use the counter-models method variation in \cite{WangWZ13,WangZZZ14} to determine the result of forgetting about atom $q$ from the following program 
\[d\la \nf c\qquad a\la q \qquad q\la b \]
while satisfying \emph{strong persistence}, the result would have 20 rules, even if strongly equivalent to the much simpler program
\[d\la \nf c\qquad a\la b\ \]
\item Even if we replace the resulting program with some equivalent canonical one at the expense of the considerable additional computational cost (cf., \cite{InoueS98,InoueS04,CabalarPV07,SlotaL11}), its syntactic form would most likely bear no real connection to the syntactic form of the original program, compromising its declarative nature. Even the syntactic form of rules that are not related to the atoms to be forgotten would likely be compromised by such method.
\item It requires the computation of HT-models, which does not come without a significant cost.
\end{itemize}

This naturally suggests the need to investigate \emph{syntactic} forgetting operators, i.e., forgetting operators that produce the result through a syntactic manipulation of the original program, thus not requiring the computation of HT-models, with the aim to keep the result as close as possible to the original program. 

Despite the fact that most research on forgetting has focused on its semantic characterisation, there are some exceptions in the literature, notably the syntactic forgetting operators described in \cite{ZhangF06,EiterW08,KnorrA14}. However, for different reasons, the early approaches are too simplistic \cite{ZhangF06,EiterW08}, in that they hardly satisfy any of the desirable properties of forgetting in ASP \cite{GoncalvesKL16}, while the recent one has a very restricted applicability \cite{KnorrA14} (cf. Sect. \ref{sec:relWork} for a detailed discussion). For example, the syntactic operators in \cite{ZhangF06} lack the semantic underpinning and are too syntax-sensitive, not even preserving existing answer sets while forgetting. The purely syntactic operator in \cite{EiterW08}, dubbed $\forget_3$, designed to preserve the answer sets, falls short of adequately distinguishing programs that, though having the same answer sets, are not strongly equivalent, which prevents it from satisfying \emph{strong persistence} even in very simple cases. Additionally, $\forget_3$ is defined for disjunctive logic programs, but may produce programs with double negation, hence preventing its iterative use. Even from a syntactic perspective, $\forget_3$ performs unnecessary changes to the initial program, which may even affect rules that do not contain atom(s) to be forgotten. Finally, the operator in \cite{KnorrA14} is only applicable to a very particular narrow non-standard subclass of programs, which excludes many cases where it is possible to forget while preserving \emph{strong persistence}, and, in general, cannot be iterated either.

In this paper, we address these issues and present a concrete forgetting operator, defined for the entire class of extended logic programs (disjunctive logic programs with double negation), that allows one to forget a single atom and produces a program in the same class. The operator is syntactic, limiting the computation of the result of forgetting to manipulating the rules in which the atom to be forgotten occurs, thus naturally yielding a resulting program that is close to the original program, while not requiring expensive computation of HT-models. Crucially, the operator satisfies \emph{strong persistence} whenever this is possible, making it the first such operator. 

Whereas one might argue that we still need to compute the HT-models in order to determine whether it is possible to satisfy \emph{strong persistence}, hence eliminating one of the advantages mentioned above, this is significantly mitigated by the fact that the operator can always be used, even in cases where \emph{strong persistence} cannot be satisfied, while still exhibiting desirable properties. In other words, if you \emph{must} forget, this operator will ensure the ``\emph{perfect}'' result whenever \emph{strong persistence} is possible, and a ``\emph{good}'' result when ``\emph{perfection}'' is not possible. In this paper, we investigate the operators' properties not only when \emph{strong persistence} is possible, but also when that is not the case. In addition, we characterise a class of programs -- dubbed \emph{$q$-forgettable} -- from which it is always possible to forget (atom $q$) while satisfying \emph{strong persistence}, whose membership can be checked in linear time.

We proceed as follows: in Sec.~\ref{sec:prelim}, we recall extended logic programs,and briefly recap important notions related to forgetting in ASP; then, we begin Sec.~\ref{sec:operator} by drawing some general considerations on syntactic forgetting and introduce some necessary notions in Sec.~\ref{subsec:syntForgetting}, introduce our new operator together with several illustrative examples in Sec.~\ref{subsec:ForgOp}, followed by a discussion of its properties in Sec.~\ref{subsec:props}; in Sec.~\ref{sec:relWork} we discuss related work; and in Sec.~\ref{sec:conclusions}, we conclude.
All proofs of the established results and a detailed comparison with relevant syntactic operators in the literature can be found in the accompanying appendix.

\section{Preliminaries}
\label{sec:prelim}
In this section, we recall necessary notions on answer set programming and forgetting. 

We assume a \emph{propositional signature} $\sign$. 
A \emph{logic program} $P$ over $\sign$ is a finite set of \emph{rules} of the form
$a_1 \vee \ldots \vee a_k  \la b_1,..., b_l, \nf c_{1},..., \nf c_m, \nf \nf d_1,..., \nf \nf d_n$,
where all $a_1,\ldots,a_k,$ $b_1,\ldots, b_l,c_{1},\ldots, c_m$, and $d_{1},\ldots, d_n$ are atoms of $\sign$. 
Such rules $r$ are also written more succinctly as 
$\head{r} \la \pbody{r}, \nf \nbody{r}, \nf \nf \nnbody{r},$
where $\head{r} = \{a_1,\ldots,a_k\}$, $\pbody{r}=\{b_1,\ldots, b_l\}$, $\nbody{r}=\{c_{1},\ldots, c_m\}$, and $\nnbody{r}=\{d_{1},\ldots, d_n\}$, and we will use both forms interchangeably. Given a rule $r$, $\head{r}$ is called the \emph{head} of $r$, and $\body{r}=\pbody{r}\cup\nf \nbody{r}\cup \nf \nf \nnbody{r}$ the \emph{body} of $r$, where, for a set $A$ of atoms, $\nf A = \{\nf q \!\!: q\in A\}$ and $\nf \nf A = \{\nf \nf q \!: q\in A\}$. 
We term the elements in $\body{r}$ \emph{(body) literals}. $\sign(P)$ and $\sign(r)$ denote the set of atoms appearing in $P$ and $r$, respectively.
The general class of logic programs we consider\footnote{Note that the use of double negation is very common in the literature of forgetting for ASP (cf.\ the recent survey \cite{GoncalvesKL16}), and in fact necessary as argued in \cite{EiterW08,WangWZ13,KnorrA14,WangZZZ14}. This does not pose a problem as double negation can be used to represent choice rules, and tools such as clingo support the syntax of double negation.} includes a number of special kinds of rules $r$: if $n=0$, then we call $r$ \emph{disjunctive}; if, in addition, $k\leq 1$, then $r$ is \emph{normal}; if on top of that $m=0$, then we call $r$ \emph{Horn}, and \emph{fact} if also $l=0$. 
The classes of \emph{disjunctive}, \emph{normal} and \emph{Horn programs} are defined resp.\ as a finite set of disjunctive, normal, and Horn rules. 
Given a program $P$ and an \emph{interpretation}, i.e., a set $I\subseteq \sign$ of atoms, the \emph{reduct} of $P$ given $I$, is defined as $P^I = \{\head{r}\la \pbody{r} : r\in P \text{ such that }\nbody{r}\cap I=\emptyset \text{ and } \nnbody{r}\subseteq I\}$.

An \emph{HT-interpretation} is a pair $\langle X,Y\rangle$ s.t.\ $X\subseteq Y \subseteq \sign$.
Given a program $P$, an HT-interpretation $\langle X,Y\rangle$ is an \emph{HT-model of $P$} if $Y\models P$ and $X\models P^{Y}$, where $\models$ denotes the standard consequence relation for classical logic.
We admit that the set of HT-models of a program $P$ is restricted to $\sign(P)$ even if $\sign(P)\subset \sign$.
We denote by $\HT(P)$ the set of \emph{all HT-models of $P$}.
A set of atoms $Y$ is an \emph{answer set} of $P$ if $\tuple{Y,Y}\in\HT(P)$, but there is no $X\subset Y$ such that $\tuple{X,Y}\in\HT(P)$.
The set of all answer sets of $P$ is denoted by $\as{P}$.
We say that two programs $P_1, P_2$ are \emph{equivalent} if $\as{P_1}=\as{P_2}$ and \emph{strongly equivalent}, denoted by $P_1\equiv P_2$, if $\as{P_1\cup R}=\as{P_2\cup R}$ for any program $R$.
It is well-known that $P_1\equiv P_2$ exactly when $\HT(P_1)=\HT(P_2)$~\cite{LPV01}. 
Given a set $V\subseteq \sign$, the \emph{$V$-exclusion} of a set of answer sets (a set of HT-interpretations) $\mathcal{M}$, denoted $\mathcal{M}_{\parallel V}$, is $\{X\text{\textbackslash} V\mid X\in\mathcal{M}\}$ ($\{\tuple{X\text{\textbackslash} V,Y\text{\textbackslash} V}\mid \tuple{X,Y}\in\mathcal{M}\}$). 

A \emph{forgetting operator} over a class $\classP$ of programs\footnote{In this paper, we only consider the very general class of programs introduced before, but, often, subclasses of it appear in the literature of ASP and forgetting in ASP.} over $\sign$ is a partial function $\fgt:\classP\times 2^{\sign}\to \classP$ s.t. the \emph{result of forgetting about $V$ from $P$}, $\f{P}{V}$, is a program over $\sign(P)\text{\textbackslash} V$, for each $P\in \classP$ and $V\subseteq \sign$. 
We denote the domain of $\fgt$ by $\classP(\fgt)$. 
The operator $\fgt$ is called \emph{closed} for $\classP'\subseteq\classP(\fgt)$ if $\f{P}{V}\in \classP'$, for every $P\in \classP'$ and $V\subseteq \sign$.
A \emph{class $\classF$ of forgetting operators (over $\classP$)} is a set of forgetting operators $\fgt$ s.t.\ $\classP(\fgt)\subseteq \classP$.

Arguably, among the many properties introduced for different classes of forgetting operators in ASP \cite{GoncalvesKL16}, \emph{strong persistence} \cite{KnorrA14} is the one that should intuitively hold, since it imposes the preservation of all original direct and indirect dependencies between atoms not to be forgotten.
In the following, $\classF$ is a class of forgetting operators.

\begin{itemize}[align=left]
\item[\pSP] $\classF$ satisfies \emph{Strong Persistence} if, for each $\fgt\in \classF$, $P\in \classP(\fgt)$ and $V\subseteq \sign$, we have $\as{\f{P}{V}\cup R}=\as{P\cup R}_{\parallel V}$, for all programs $R\in \classP(\fgt)$ with $\sign(R)\subseteq \sign\setm V$. 
\end{itemize}
Thus, \pSP\ requires that the answer sets of $\fgt(P,V)$ correspond to those of $P$, no matter what programs $R$ over $\sign\setm V$ we add to both, which is closely related to the concept of strong equivalence. 
Among the many properties implied by \pSP\ \cite{GoncalvesKL16}, \pSI\ indicates that rules not mentioning atoms to be forgotten can be added before or after forgetting.

\begin{itemize}[align=left]
 \item[\pSI] $\classF$ satisfies \emph{Strong (addition) Invariance} if, for each $\fgt\in \classF$, $P\in \mathcal{C}$ and $V\subseteq \sign$, we have $\f{P}{V}\cup R \equiv \f{P\cup R}{V}$ for all programs $R\in \classP$ with $\sign(R)\subseteq \sign\text{\textbackslash} V$.
\end{itemize}
 
However, it was shown in~\cite{GoncalvesKL-ECAI16} that 
there is no forgetting operator that satisfies \pSP\ and that is defined for all pairs $\tuple{P,V}$, called \emph{forgetting instances}, where $P$ is a program and $V$ is a set of atoms to be forgotten from $P$.
Given this general impossibility result, \pSP\ was defined for concrete forgetting instances.

A forgetting operator $\fgt$ over $\classP$ satisfies \pSP$_{\tuple{P,V}}$, for $\tuple{P,V}$ a forgetting instance over $\classP$, if $\as{\f{P}{V}\cup R}=\as{P\cup R}_{\parallel V}$, for all programs $R\in \classP$ with $\sign(R)\subseteq \sign\setm V$. 

A sound and complete criterion $\Omega$ was presented to characterize 
when it is not possible to forget while satisfying \pSP$_{\tuple{P,V}}$. 
An instance $\tuple{P,V}$ \emph{satisfies criterion $\Omega$} if there is $Y\subseteq \sign\text{\textbackslash} V$ such that the set of sets 
$ \SRel^Y=\{\RA_{\tuple{P,V}}^{Y,A}\mid A\in \Rel_{\tuple{P,V}}^Y\}$
 is non-empty and has no least element, where 
\begin{align*}
\RA^{Y,A}_{\tuple{P,V}} & =\{X\text{\textbackslash} V\mid \tuple{X,Y\cup A}\in \HT(P)\} \\ 
 \Rel_{\tuple{P,V}}^Y & =\{A\subseteq V\mid \tuple{Y\cup A,Y\cup A}\in \HT(P) \text{ and }
  \nexists A'\subset A \text{ s.t.\ }\tuple{Y\cup A',Y\cup A}\in \HT(P)\}.
 \end{align*}

This technical criterion was shown to be sound and complete, i.e., it is not possible to forget about a set of atoms $V$ from a program $P$ exactly when $\tuple{P,V}$ satisfies $\Omega$.
A corresponding class of forgetting operators, $\classF_{\spF}$, was introduced \cite{GoncalvesKL-ECAI16}.
\begin{align*}
\classF_{\spF} = \{\fgt\mid \HT(\f{P}{V}) \! = \! \{\tuple{X,Y}\mid Y\subseteq \sign(P)\text{\textbackslash} V & \wedge X\!\in \bigcap\SRel^Y \},
\text{ for all } P\in \classP(\fgt) \text{ and } V\subseteq\sign\}.
\end{align*}
It was shown that every operator in $\classF_{\spF}$ satisfies \pSP$_{\tuple{P,V}}$ for instances ${\tuple{P,V}}$ that do not satisfy  $\Omega$.
Moreover, in the detailed study of the case where $\Omega$ is satisfied~\cite{GoncalvesKLW17}, it was shown that the operators in $\classF_{\spF}$ preserve all answer sets, even in the presence of an additional program without the atoms to be forgotten. This makes   $\classF_{\spF}$ an ideal choice to forgetting while satisfying \pSP\ whenever possible. 
Whereas $\classF_{\spF}$ is only defined semantically, i.e., it only specifies the HT-models that a result of forgetting a set of atoms from program should have, a specific program could be obtained from that set of HT-models based on its counter-models \cite{CabalarF07} -- a construction previously adapted for computing concrete results of forgetting for classes of forgetting operators based on HT-models \cite{WangWZ13,WangZZZ14}. 

\section{A Syntactic Operator}
\label{sec:operator}

In this section, we present a syntactic forgetting operator $\fSPOP$ that satisfies \pSP\ whenever possible.
We start with some general considerations on syntactic forgetting and introduce some necessary notions in Sec.~\ref{subsec:syntForgetting}.
Then, we introduce our new operator $\fSPOP$ together with explanatory examples in Sec.~\ref{subsec:ForgOp}, followed by a discussion of its properties in Sec.~\ref{subsec:props}.

\subsection{On Syntactic Forgetting in ASP}\label{subsec:syntForgetting}

One fundamental idea in syntactic forgetting (in ASP) is to replace the occurrences of an atom to be forgotten in the body of some rule with the body of a rule whose head is the atom to be forgotten.
This can be rooted in the weak generalized principle of partial evaluation wGPPE \cite{BrassD99} and aligns with the objective to preserve answer sets, no matter which rules over the remaining atoms are to be added.

\begin{example}\label{ex:simpleForgetting}
Consider program $P=\{t \la q; v \la \nf q; q \la s; q \la w\}$ from which we want to forget about $q$.
We claim that the following should be a result of forgetting about $q$.
\begin{align*}
t & \la s & t & \la w & v & \la \nf s, \nf w
\end{align*}
Since $t$ depends on $q$, and $q$ depends on $s$ as well as on $w$, we want to preserve these dependencies without mentioning $q$, thus creating two rules in which $q$ is replaced by $s$ and $w$, respectively.
This way, whenever a set $R$ of rules allow us to derive $s$ (or $w$), then $t$ must occur in all answer sets of $P\cup R$ as well as of $\fgt(P,\{q\})\cup R$. 
At the same time, since $v$ depends on $\nf q$, and $q$ would be false whenever both $s$ and $w$ are false, we capture this in a rule that represents this dependency. 
\end{example}

As we will see later, these natural ideas provide the foundation for the syntactic operator we are going to define, even though several adjustments have to be made to be applicable to all programs, and to ensure that the operator satisfies \pSP\ whenever possible. 
Still, even in such simplified examples, the occurrence of certain rules that can be considered redundant would complicate this idea of syntactic forgetting.
If $q\la q$ occurred in program $P$, then we would certainly not want to replace $q$ by $q$ and add $t\la q$. 
To avoid the problems caused by redundant rules and redundant parts in rules, we simplify the program upfront, and similarly to \cite{KnorrA14} and previous related work \cite{InoueS98,InoueS04,CabalarPV07,SlotaL11}, we adopt a normal form that avoids such redundancies, but is otherwise syntactically identical to the original program. 
There are two essential differences to the normal form considered in \cite{KnorrA14}. First, contrarily to \cite{KnorrA14}, our normal form applies to programs with disjunctive heads. Moreover, we eliminate non-minimal rules \cite{BrassD99}, which further strengthens the benefits of using normal forms.

Formally, a rule $r$ in $P$ is \emph{minimal} if there is no rule $r'\in P$ such that $\head{r'}\subseteq \head{r}\wedge \body{r'}\subset \body{r}$ or 
$\head{r'}\subset \head{r}\wedge \body{r'}\subseteq \body{r}$.
We also recall that a rule $r$ is \emph{tautological} if $\head{r}\cap\pbody{r}\not=\emptyset$, or $\pbody{r}\cap\nbody{r}\not=\emptyset$, or $\nbody{r}\cap\nnbody{r}\not=\emptyset$.

\begin{definition}\label{def:normalForm}
Let $P$ be a program.
We say that $P$ is in \emph{normal form} if the following conditions hold:
\begin{itemize}
\item for every $a\in \sign(P)$ and $r\in P$, at most one of $a$, $not\ a$ or $not\ not\ a$
is in $B(r)$;
\item if $a\in H(r)$, then neither $a$, nor $not\ a$ are in $B(r)$;
\item all rules in $P$ are minimal.
\end{itemize}
\end{definition}
The next definition shows how to transform any program into one in normal form.

\begin{definition}\label{def:constructionNormalForm}
Let $P$ be a program. The normal form $\NF{P}$ is obtained from $P$ by:
\begin{enumerate}
\item removing all tautological rules; 
\item removing all atoms $a$ from $\nnbody{r}$ in the remaining rules $r$, whenever $a\in\pbody{r}$;
\item removing all atoms $a$ from $\head{r}$ in the remaining rules $r$, whenever $a\in \nbody{r}$;
\item removing from the resulting program all rules that are not minimal.
\end{enumerate}
\end{definition}

\noindent
Note that the first item of Def.~\ref{def:normalForm} is ensured by conditions 1 and 2 of Def.~\ref{def:constructionNormalForm}, the second by 1 and 3, and the third by condition 4.
We can show that the construction of $\NF{P}$ is correct.

\begin{proposition}\label{prop:normalForm}
Let $P$ be a program. 
Then, $\NF{P}$ is in normal form and is strongly equivalent to $P$.
\end{proposition}
In addition, $\NF{P}$ can be computed in at most quadratic time in terms of the number of rules in $P$ (as ensuring minimality requires comparing all $n$ rules with each other).
\begin{proposition}\label{prop:compNormalForm}
Let $P$ be a program. 
Then, the normal form $\NF{P}$ can be computed in PT{\scriptsize IME}.
\end{proposition}
Thus for the remainder of the paper, we only consider programs in normal form, as this can be efficiently computed and is syntactically equal to the original program apart from redundancies.

Coming back to Ex.~\ref{ex:simpleForgetting}, the way we replaced $\nf q$ becomes more complicated when the rules with head $q$ have more than one body literal or head atoms other than $q$.

\begin{example}\label{ex:as-dual}
Consider program $P=\{v \la \nf q; q \la s, t; q\vee u \la w\}$, a variation of the program in Ex.~\ref{ex:simpleForgetting}.
We observe that the following rules for $v$ would desirably be in the result of forgetting $q$.
\begin{align*}
v &\la \nf s,\nf w  & v &\la \nf t,\nf w    & v &\la \nf s,\nf\nf u   & v &\la \nf t,\nf\nf u 
\end{align*}

These four rules correspond to the four possible ways to guarantee that $q$ is false, i.e., that $q$ is not derived by any of the rules with $q$ in the head. For each rule, this means  that either one body literal is false or one head atom different from $q$ is true. The first two rules correspond to the cases where one body atom in each rule with head $q$ is false. 
The last two rules correspond to the case where $u$, the other disjunct in the head of the third rule, is true and one of $s$ and $t$ is false.
Intuitively, $u$ in the head corresponds to $\nf u$ in the body, which is why $\nf\nf u$ appears.

Besides the rules for $v$, we should also have the following rules for deriving $u$.
\begin{align*}
 u &\la w,\nf s,\nf\nf u 	& u &\la w,\nf t,\nf\nf u
\end{align*}
These rules are necessary to guarantee that, similarly as in $P$, $u$ can be either true or false (which is why $\nf\nf u$ appears) whenever $w$ (the body of the original rule with head $u$) is true.

Together, these six rules guarantee the preservation of the existing implicit dependencies of $P$.
\end{example}
In \cite{EiterW08,KnorrA14}, sets of conjunctions of literals are collected with the aim of replacing some $\nf q$ while preserving its truth value (though not considering disjunction nor double negation \cite{EiterW08}).
We now introduce the notion of \emph{as-dual} $\asdual{q}{P}$, as generalized from \cite{KnorrA14}, which collects the possible ways to satisfy all rules of $P$ independently of $q$.
The intuitive idea is that, when applied to the set of rules that contain $q$ in the head, the as-dual contains sets of literals, each representing a possible way to guarantee that $q$ is false (since every rule with $q$ in the head is satisfied independently of $q$).

For this purpose, we need to introduce further notation.
For a set $S$ of literals, $\nf(S)=\{\nf s:s\in S\}$ and $\nf\nf(S)=\{\nf\nf s :s\in S\}$, where, for $p\in \sign$, we assume the simplification $\nf\nf\nf p=\nf p$ and $\nf\nf\nf\nf p=\nf\nf p$.
The sets $B^{\setminus q}(r)$ and $H^{\setminus q}(r)$ respectively denote the set of body and head literals after removing every occurrence of $q$, i.e., $B^{\setm q}(r)=B(r)\setm\{q,\nf q,\nf\nf q\}$ and $H^{\setm q}(r)=H(r)\setm\{q\}$.
We define $\asdual{q}{P}$ as follows.
\begin{align*}
\asdual{q}{P}=\{&\nf(\{l_1,\dots,l_m \})\ \cup\nf\nf(\{l_{m+1},\dots,l_n\}):l_i\in B^{\setminus q}(r_i),1\leq i\leq m,\\
&l_j\in H^{\setminus q}(r_j),m+1\leq j\leq n,
\langle\{r_1,\dots,r_m\},\{r_{m+1},\dots,r_n\}\rangle
\text{ is a partition of }P\}
\end{align*}
The as-dual construction considers the possible partitions of $P$, and for each partition $\langle F,T\rangle$ collects the negation of exactly one element (except $q$) from the body of each rule of $F$, thus guaranteeing that the body of every rule of $F$ is not satisfied, together with the double negation of exactly one head atom (except $q$) from each rule of $T$, thus guaranteeing that the head of every rule of $T$ is satisfied, together guaranteeing that all rules in $P$ are satisfied independently of $q$.

When applied to the set of rules that contain $q$ in the head, the as-dual represents the possible ways to replace the occurrences of $\nf q$ in the body of a rule when forgetting about $q$ from $P$. 
 It is important to note that, in this case, each set in $\asdual{q}{P}$ contains one literal for each rule that contains $q$ in the head (those that allow $q$ to be derived), since only in this way we can guarantee that $q$ is false. 
This definition covers two interesting corner cases. First, if there is no rule with $q$ in its head, i.e., the input program $P$ is empty, then $\asdual{q}{P}=\{\emptyset\}$, meaning that $q$ is false, and therefore we do not need to impose conditions on other atoms. Finally, if $P$ contains $q$ as a fact, then $\asdual{q}{P}=\emptyset$, since in this case it is impossible to make $q$ false.

\begin{example}\label{ex:asDualConstructionExample}
Consider the program of Ex.~\ref{ex:simpleForgetting}. Applying the as-dual to the two rules with $q$ in the head, we obtain $\{\{\nf s,\nf w\}\}$, whose unique element corresponds to the only way to make $q$ false.

If we now consider the program of Ex.~\ref{ex:as-dual}, which has two rules with $q$ in the head, the as-dual construction renders $\{\{\nf s,\nf\nf u\},\{\nf t,\nf\nf u\},\{\nf s,\nf w\},\{\nf t,\nf w\}\}$, which correspond to the four possible ways to guarantee that $q$ is false.
\end{example}

\subsection{A Novel Forgetting Operator}\label{subsec:ForgOp}

We are now ready to present the formal definition of the operator $\fSPOP$.
As this definition is technically elaborate, we will first present the new operator itself that allows forgetting a single atom from a given program, and subsequently explain and illustrate its definition.

\begin{definition}[Forgetting about $q$ from $P$]\label{def:OurForgetting}
Let $P$ be a program over $\Sigma$, and $q\in \Sigma$. Let $P'=\NF{P}$ be the normal form of $P$.
Consider the following sets, each representing a possible way $q$ can appear in rules of $P'$:
\begin{align*}
R_\text{ }&:= \{r\in P'\mid q\not\in \Sigma(r)\} 
&
R_2&:= \{r\in P'\mid not\ not\ q \in B(r), q\not\in H(r)\}
\\
R_0&:= \{r\in P'\mid q\in B(r)\}
&
R_3&:= \{r\in P'\mid not\ not\ q \in B(r), q\in H(r)\}
\\
R_1&:= \{r\in P'\mid not\ q\in B(r)\}
&
R_4&:= \{r\in P'\mid not\ not\ q \not\in B(r), q\in H(r)\}
\end{align*}
The result of forgetting about $q$ from $P$, $\fgt_{SP}(P,q)$, is $\NF{P''}$, where $P''$ is as follows: 
\begin{itemize}
\item each $r\in R$
\item for each $r_0\in R_0$
\begin{description}
    \item[1a] for each $r_4\in R_4$\label{1a}\\
        $\head{r_0}\cup \headwoq{r_4} \leftarrow \bodywoq{r_0}\cup \body{r_4}$
    \item[2a] for each $r_3\in R_3$, $r'\in R_1\cup R_4$\label{2a}\\
        $\head{r_0}\cup \headwoq{r_3} \leftarrow \bodywoq{r_0}\cup \bodywoq{r_3}
        \cup \nf(\headwoq{r'})\cup \nf\nf(\bodywoq{r'})$
    \item[3a] for each $r_3\in R_3$, $h(r_0)\in \head{r_0}$, $D\in \asdual{q}{R_0\cup R_2 \setminus \{r_0\}}$\label{3a}\\
        $\head{r_0} \leftarrow \bodywoq{r_0}\cup \{not\ not\ h(r_0))\}\cup D\cup \bodywoq{r_3}\cup \nf(\headwoq{r_3})$
\end{description}
\item for each $r_2\in R_2$
\begin{description}
    \item[1b] for each $r_4\in R_4$\label{1b}\\
        $\head{r_2} \leftarrow \bodywoq{r_2}\cup \nf(\headwoq{r_4})\cup \nf\nf(\body{r_4})$
    \item[2b] for each $r_3\in R_3$, $r'\in R_1\cup R_4$\label{2b}\\
        $\head{r_2} \leftarrow \bodywoq{r_2}\cup \nf(\headwoq{r_3}\cup \headwoq{r'})\cup \nf\nf(\bodywoq{r_3}\cup \bodywoq{r'})$
    \item[3b] for each $r_3\in R_3$, $h(r_2)\in \head{r_2}$, $D\in \asdual{q}{R_0\cup R_2 \setminus \{r_2\}}$\label{3b}\\
        $\head{r_2} \leftarrow \bodywoq{r_2}\cup \nf(\headwoq{r_3})\cup \nf\nf (\bodywoq{r_3}\cup \{h(r_2)\})\cup D$
\end{description}
\item for each $r'\in R_1\cup R_4$
\begin{description}
    \item[4] for each $D\in \asdual{q}{R_3\cup R_4}$ such that $D\cap \nf \bodywoq{r'}=\emptyset$ \label{4}\\
        $\headwoq{r'} \leftarrow \bodywoq{r'}\cup D$
    \item[5] for each $r_3\in R_3$, $r\in R_0\cup R_2$, $D\in \asdual{q}{R_4}$ such that $D\cap \nf \bodywoq{r'}=\emptyset$ \label{5}\\
        $\headwoq{r'} \leftarrow \bodywoq{r'}\cup \nf(\head{r}\cup \headwoq{r_3})\cup \nf\nf(\bodywoq{r}\cup \bodywoq{r_3})\cup D$
    \item[6] for each $r_3\in R_3$, $h(r')\in \head{r'}$, $D\in \asdual{q}{R_1\cup R_4\setminus\{r'\}}$\label{6}\\
        $\headwoq{r'}\leftarrow \bodywoq{r'}\cup \nf(\headwoq{r_3}) \cup \nf\nf(\bodywoq{r_3}\cup \{h(r')\})\cup D$
\end{description}
\item for each $r_0\in R_0$
\begin{description}
    \item[7] for each $r_3,r'_3\in R_3$ with $r_3\not=r'_3$, $D\in \asdual{q}{R_0\cup R_2\setminus\{r_0\}}, h(r_0)\in H(r_0)$\label{7}\\
    $H(r_0)\cup H^{\setminus q}(r_3)\leftarrow B^{\setminus q}(r_0)\cup B^{\setminus q}(r_3)\cup \nf(H^{\setminus q}(r'_3)) \cup \nf\nf(B^{\setminus q}(r'_3)\cup \{h(r_0)\}) \cup D$
\end{description}
\end{itemize}
\end{definition}

Using the normal form $P'$ of $P$, five sets of rules, $R_0$, $R_1$, $R_2$, $R_3$, and $R_4$, are defined over $P'$, in each of which $q$ appears in the rules in a different form. 
In addition, $R$ contains all rules from $P'$ that do not mention $q$.
These appear unchanged in the final result of forgetting.

The construction is divided in two cases: one for the rules which contain $q$ or $\nf\nf q$ in the body ($R_0$ or $R_2$), and one for the rules that contain $\nf q$ in the body or $q$ in the head ($R_1$ or $R_4$).

For rules $r$ in $R_0\cup R_2$, \textbf{1a} and \textbf{1b} generate the rules obtained by substituting the occurrences of $q$ or $\nf\nf q$ by the body of those rules whose head is $q$ and do not contain $q$ in the body (those in $R_4$). 
Also, \textbf{2a} and \textbf{2b} create the model-generating rules, one for each rule $r'$ supported by $\nf q$ (those in $R_1\cup R_4$), and \textbf{3a} and \textbf{3b} create rules of the form $\head{r}\la\nf\nf \head{r}, \body{r}$, one such rule for every possible partition of the bodies of rules in $R_0\cup R_2$, which can be obtained using the as-dual construction.
Each such rule also contains 
$\redBody{r}{q}$, $\redBody{r_3}{q}$ and $\nf(\redHead{r_3}{q})$ in the body, because the original rule $\head{r}\la q$ is triggered if $\redBody{r}{q}$ and $\redBody{r_3}{q}$ are true and no other head atom but $q$ in the generating rule $r_3$ is true. 

Now let us consider the other case (for rules $r_1$ in $R_1$ and $r_4$ in $R_4$).
Since $\nf q$ appears in the body of $r_1$, the case \textbf{4} is based on the as-dual as discussed in Ex.~\ref{ex:as-dual}. The idea is that $\head{r_1}$ can be concluded if $\redBody{r_1}{q}$ is true and no body of a rule with head $q$ (those in $R_3$ and $R_4$) is true. Likewise, $\redHead{r_4}{q}$ can be concluded if $\body{r_4}$ is true and there is no further evidence that $q$ is true, which again can be given by the as-dual construction.
The cases \textbf{5} and \textbf{6} are similar to those of \textbf{2a}, \textbf{2b} and \textbf{3a}, \textbf{3b}. If there is cyclic support for $q$, and each head of $R_1\cup R_4$ is true, we can justify cyclic support for the rule. For each of these rules that has no true head, we require a false body atom as evidence that the head is not true, because its body is not satisfied.
Lastly, if there are multiple self-cycles on the atom that is to be forgotten, the case \textbf{7} connects them so that they are handled correctly.

Once all these rules are computed, a final normalization step is applied to remove any tautologies or irrelevant atoms in the resulting rules.

\begin{example}
Recall $P$ from Ex.~\ref{ex:simpleForgetting}.
The sets $R_2$ and $R_3$ are empty and, therefore,
the result of $\fgt_{SP}(P,q)$ is
\begin{align*}
t & \la s & t & \la w & v & \la \nf s, \nf w
\end{align*}
where the first two rules are produced by \textbf{1a}, s.t.\ $q$ in the body of $t\la q$ is replaced by $s$ and $w$, resp. The third rule is obtained by \textbf{4}, which replaces $\nf q$ in the body of $v\la\nf q$ by $\nf s,\nf w$ since $\asdual{q}{R_3\cup R_4}=\{\{\nf s,\nf w\}\}$.
The result corresponds exactly to the one given in Ex.~\ref{ex:simpleForgetting}.
\end{example}

\begin{example}
Recall program $P$ of Ex.~\ref{ex:as-dual} which also uses disjunction.
Since $\asdual{q}{R_3\cup R_4}=\{\{\nf s,\nf\nf u\},$ $\{\nf t,\nf\nf u\},\{\nf s,\nf w\},\{\nf t,\nf w\}\}$, \textbf{4} replaces $q$ in the head of $q\vee u\la w$, and $\nf q$ in the body of $v\la\nf q$ with such elements, yielding 
\begin{align*}
v &\la \nf s,\nf w      & v &\la \nf s,\nf\nf u   & u &\la w,\nf s,\nf\nf u \\
v &\la \nf t,\nf w      & v &\la \nf t,\nf\nf u   & u &\la w,\nf t,\nf\nf u \;
\end{align*}
the result of $\fgt_{SP}(P,q)$. Note that by \textbf{4}, $\{\nf s,\nf w\}$ and $\{\nf t,\nf w\}$ are not used with $q\vee u \la w$.
\end{example}

In the previous examples, no cyclic dependencies on the atom to be forgotten existed. Next, we discuss two examples that deal with this more sophisticated case.

\begin{example}\label{ex:simple_choice}
Consider forgetting about $q$ from the following program $P=\{q\la\nf\nf q; a \la q\}$, an example which cannot be handled by any syntactic forgetting operator in the literature.
In this case, since $R_1$, $R_2$ and $R_4$ are empty, the only applicable rule is \textbf{3a}. Since there is a cyclic dependency on $q$, i.e., $R_3\neq \emptyset$, the application of  \textbf{3a} renders the result of $\fgt_{SP}(P,q)$ to be:
\begin{align*}
a\la\nf\nf a
\end{align*}
Thus, when forgetting about $q$ from $P$, the cyclic dependency on $q$ is transferred to $a$.
\end{example}
If there are several dependencies on $q$ though, such self-cycles also create an implicit dependency between the elements supported by $q$ (the heads of rules in $R_0\cup R_2$) and those supported by $\nf q$ (the heads of rules in $R_1\cup R_4$). When forgetting $q$, such dependencies must be taken into account.

\begin{example}\label{example-choice}
Now, consider program $P=\{q\la \nf\nf q; u\la q; s\la q; t \la \nf q\}$.
The answer sets $\{q,u,s\}$ and $\{t\}$ of $P$ should be preserved (for all but $q$) when forgetting about $q$, even if we add to $P$ a set of rules not containing $q$.
By distinguishing which rule heads depend on $q$ (those of $R_0\cup R_2$) and which on $\nf q$ (those of $R_1\cup R_4$), \textbf{2a} and \textbf{5} create the following model-generating rules, which guarantee that the atoms that depend on $q$ should be true whenever the atoms that depend on $\nf q$ are false, and vice versa:
\begin{align*}
u \la \nf t \qquad s \la \nf t \qquad t \la \nf u \qquad t \la \nf s
\end{align*}
These rules alone have the two desired answer sets $\{u,s\}$ and $\{t\}$, since they basically state that either the elements supported by $q$ are true, or the elements supported by $\nf q$ are, but not both. 
In addition, \textbf{3a} and \textbf{6} add the following rules:
\begin{align*}
u\la \nf\nf u,\nf\nf s \qquad s \la \nf\nf s,\nf\nf u \qquad t\la \nf\nf t
\end{align*}
The first two rules guarantee that, whenever $t$ is derivable (independently, e.g., by the existence of $t\la$), then $u$ and $s$ may both either be simultaneously true or false. Similarly, the third rule ensures that $t$ may still vary between true and false if we, e.g., add both $u\la$ and $s\la$.
\end{example}

\subsection{Properties}\label{subsec:props}
In this section, we present several properties that our syntactic operator $\fSPOP$ satisfies.

First of all, we show that $\fSPOP$ is in fact a forgetting operator, i.e., $\fSPOP(P,q)$ does not contain $q$.

\begin{proposition}\label{prop:existence}
Let $P$ be a program over $\Sigma$ and $q\in \Sigma$. Then $\fSP{P}{q}$ is a program over $\Sigma\setm\{q\}$.
\end{proposition}

In order to show that $\fSPOP$ satisfies \pSP\ whenever this is possible, we start by showing that $\fSPOP$ semantically coincides with any operator of the class $\FSPOP$ (cf. Def.~4 of~\cite{GoncalvesKL-ECAI16}). 

\begin{theorem}\label{thm:FSPsingleVariable}
Let $P$ be a program over $\Sigma$ and $q\in \Sigma$. For any $\fgt\in \FSPOP$, we have  
$\HT(\fSP{P}{q})=\HT(\f{P}{\{q\}})$.
\end{theorem}

Since, as mentioned in Sec.~\ref{sec:prelim}, $\FSPOP$ was shown to correspond to the class of forgetting operators that satisfy \pSP\ for all those cases when this is possible (cf. Thm.~4 of~\cite{GoncalvesKL-ECAI16}), Thm.~\ref{thm:FSPsingleVariable} allows us to conclude the main result of the paper, i.e., that $\fSPOP$ satisfies \pSP\ for all those cases when this is possible. In fact, it is the first such syntactic forgetting operator.  

\begin{theorem}\label{thm:fSPsatSP}
Let $P$ be a program over $\Sigma$ and $q\in \Sigma$. If $\tuple{P,\{q\}}$ does not satisfy $\Omega$, then $\fSPOP$ satisfies \pSP$_{\tuple{P,\{q\}}}$. 
\end{theorem}

This result guarantees that $\fSPOP$ provides the ideal result whenever $\Omega$ is not true. 

When $\Omega$ is true, such as in Ex.~\ref{example-choice}, although we are not in an ideal situation, it might be the case that we must forget~\cite{GoncalvesKLW17}. In this case, we can use $\fSPOP$, as it is defined for every extended program, while still satisfying several desirable properties.

Property \pSI\ guarantees that all rules of a program $P$ not mentioning $q$ be (semantically) preserved when forgetting about $q$ from $P$. Notably, although \pSI\ is a desirable property, several classes of forgetting operators in the literature fail to satisfy it. Interestingly, $\fSPOP$ preserves all rules of $P$ not mentioning $q$, thus satisfying a strong version of \pSI.
\begin{proposition}\label{prop:FSPsatSI}
Let $P$ be a program over $\Sigma$ and $q\in \Sigma$. Then,
$\fSP{P\cup R}{q}=\fSP{P}{q}\cup R$, for all programs $R$ over $\Sigma\setm\{q\}$.
\end{proposition}

An important property that $\fSPOP$ satisfies in general is the preservation of all answer sets of $P$ (modulo $q$), even in the presence of an additional program not containing $q$.

\begin{theorem}\label{prop:FSPsatwSP}
Let $P$ be a program over $\Sigma$ and $q\in \Sigma$. Then,
$\as{P\cup R}_{\parallel \{q\}}\subseteq \as{\fSP{P}{q}\cup R}$, for all programs $R$ over $\sign\setm \{q\}$.
\end{theorem}
Obviously, whenever $\fSPOP$ satisfies \pSP$_{\tuple{P,\{q\}}}$, then we obtain equality of the two sets of answer sets in the previous result.
However, when $\Omega$ is satisfied, then we may obtain a strict superset.
Take Ex.~\ref{example-choice}, where $\Omega$ is satisfied. We can observe that the provided result does admit a third answer set $\{s,t,u\}$ besides the two $P$ itself has.
But, no matter which program $R$ we add, existing answer sets are preserved.
In other words, even if $\Omega$ is satisfied, the construction of $\fSPOP$ ensures that no existing answer sets (modulo the forgotten atom) are lost while forgetting. 

The operator $\fSPOP$ also preserves strong equivalence, i.e., forgetting the same atom from two strongly equivalent programs yields strongly equivalent results. 

\begin{proposition}\label{prop:FSPsatSE}
Let $P$ and $P'$ be programs over $\Sigma$ and $q\in \Sigma$. If $P\equiv P'$ then $\fSP{P}{q}\equiv \fSP{P'}{q}$.
\end{proposition}
These positive results suggest that, in cases when we \emph{must} forget, $\fSPOP$ can be used without first checking $\Omega$, in line with the observations on the usage of $\FSPOP$ in \cite{GoncalvesKLW17}.

Still, we may want to find a broad class of instances ${\tuple{P,\{q\}}}$ for which we can guarantee that \pSP$_{\tuple{P,\{q\}}}$ is satisfied without having to check criterion $\Omega$. 
We have seen in Ex.~\ref{example-choice} that self-cycles on the atom to be forgotten are relevant for forgetting while satisfying \pSP, a fact already observed in \cite{KnorrA14}. 
We now extend the notion of $q$-forgettable given in \cite{KnorrA14} to programs with disjunction in the head of the rules.

\begin{definition}\label{def:qfrgtbl}
Let $P$ be a program in normal form over $\sign$ and $q\in\sign$. Then, we say that $P$ is \emph{q-forgettable} if at least one of the following conditions holds:
\begin{itemize}
\item all occurrences of $q$ in $P$ are within self-cycles, rules with $q\in\head{r}$ and $q\in\nnbody{r}$
\item $P$ contains the fact $q\la$
\item $P$ contains no self-cycle on $q$, a rule with $q\in\head{r}$ and $q\in\nnbody{r}$
\end{itemize}
\end{definition}

We can show that, when restricted to $q$-forgettable programs, $\Omega$ is not satisfied.

\begin{theorem}\label{thm:qforgetable_implies_sp}
Let $P$ be a program over $\Sigma$, and $q\in \Sigma$. If $P$ is $q$-forgettable, then $\langle P,\{q\}\rangle$ does not satisfy $\Omega$.
\end{theorem}

An immediate consequence of this theorem is that, for $q$-forgettable programs, $\fSPOP$ satisfies \pSP$_{\tuple{P,\{q\}}}$.
In the case of the programs of Ex.~\ref{ex:simpleForgetting} and Ex.~\ref{ex:as-dual}, which are $q$-forgettable, we can forget $q$ from both programs while satisfying \pSP$_{\tuple{P,\{q\}}}$, and we can use $\fSPOP$ to obtain the desired result.

The converse of Thm.~\ref{thm:qforgetable_implies_sp} does not hold in general. Program $P$ of Ex.~\ref{ex:simple_choice} is a simple counter-example, since it is not $q$-forgettable, and $\tuple{P,\{q\}}$ does not satisfy $\Omega$. 
This is not surprising given that deciding $\Omega$ is $\Sigma^P_3$-complete (cf. Thm.~7 of \cite{GoncalvesKLW17}), whereas deciding if a program is $q$-forgettable requires to check each rule once.
\begin{proposition}\label{prop:checking_q-forgettable_complexity}
Let $P$ be a program over $\Sigma$, and $q\in \Sigma$. Deciding if $P$ is $q$-forgettable can be done in linear time.
\end{proposition}

Besides the simplicity of checking if a program is $q$-forgettable, the restriction to these programs implies that $\fSPOP$ can be constructed using only a small subset of the derivation rules. 

\begin{theorem}\label{def:OurForgettingSimple}
Let $P$ be a program over $\Sigma$, and $q\in \Sigma$.
If $P$ is $q$-forgettable, then $\fSP{P}{q}$ is constructed using only the derivation rules \textbf{1a}, \textbf{1b} and \textbf{4}.
\end{theorem}

We can now present the complexity result of our forgetting operator $\fSPOP$.

\begin{theorem}\label{thm:complexity}
Let $P$ be a program over $\Sigma$ and $q\in\Sigma$. Then, computing $\fSP{P}{q}$ is in EXPT{\scriptsize IME} in the number of rules containing occurrences of $q$ and linear in the remaining rules.
\end{theorem}

This result is not surprising taking into account the arguments in~\cite{EiterK2018} showing that the result of forgetting is necessarily exponential. 
Nevertheless, in the case of $\fSPOP$, this exponential behavior is limited to the rules mentioning the atom to be forgotten.

Besides the properties already mentioned, a fundamental characteristic of the operator $\fSPOP$ is its syntactic nature, in the sense that the result of forgetting is obtained by a manipulation of the rules of the original program. As a consequence, the result of forgetting according to $\fSPOP$ is somehow close to the original program. 
In order to formalize this idea, we define a distance function between programs. It builds on a distance measure between rules, which counts the number of literals only appearing in one of the rules.

\begin{definition}
Let $r$ and $r'$ be two rules over $\Sigma$. The distance  between $r$ and $r'$ is
$d(r,r')=\lvert\head{r}\ominus\head{r'}\rvert + \lvert\body{r}\ominus\body{r'}\rvert$ where $A\ominus B=(A\setm B)\cup (B\setm A)$ is the usual symmetric difference.
The size of a rule $r$ is defined as $\lvert r\rvert=\lvert\head{r}\rvert+\lvert\body{r}\rvert$.
\end{definition}

To extend this distance to a distance between programs, we use mappings between programs.

\begin{definition}
Let $P_1$ and $P_2$ be programs over $\Sigma$.
The distance between $P_1$ and $P_2$ is 
$dist(P_1,P_2)=Min\{dist_m(P_1,P_2): m \text{ is a mapping between }P_1 \text{ and } P_2\}$, 
where a mapping between $P_1$ and $P_2$ is a partial injective function $m:P_1\to P_2$ and 
$dist_{m}(P_1,P_2)=Sum\{d(r,m(r)): m(r)\in P_2\}+Sum\{\lvert r\rvert: r\in (P_1\setm {m}^{-1}[P_2])\}+ Sum\{\lvert r\rvert: r \in (P_2\setm m[P_1])\}$.
\end{definition}

The distance between $P_1$ and $P_2$ induced by a mapping $m$ is the sum of the distances of those rules associated by $m$, and the sizes of the remaining rules. The distance between $P_1$ and $P_2$ is then the minimal value for the possible mappings of $P_1$ and $P_2$.
Intuitively, this distance corresponds to the minimal number of literals we need to add to both programs to make them equal, noting that we may need to add new rules.

\begin{example}
Let $P_1=\{a\la b, \nf c \}$ and $P_2=\{a\la \nf c,\ b\la d\}$ be two programs. Then, $dist(P_1,P_2)$=3, which corresponds to the sum of the distance between the first rule of $P_2$ and the rule of $P_1$, which is 1, with 
the size of $b\la d$, which is 2. Intuitively, this corresponds to adding $b$ to the body of the first rule of $P_2$ and adding the entire rule $b\la d$ to $P_1$. We could consider adding $P_1$ to $P_2$ and $P_2$ to $P_1$, corresponding to the mapping that is undefined for each rule of $P_1$, but this would induce a distance of 7, the sum of the size of every rule in the programs. In fact, this is always an upper bound to the distance between two programs.
\end{example}

Let us now consider the distance function in the case of forgetting. We compare our syntactic operator $\fSPOP$ with a semantic operator $\fSem\in\FSPOP$, which is defined using the counter-model construction (based on \cite{WangWZ13,WangZZZ14}).

First, we prove a result regarding the rules generated by the semantic operator.

\begin{proposition}\label{prop:sizeOfRulesFSem}
Let $P$ be a program over $\sign$ and $q\in \Sigma$.
Then, for each $r\in \fSem(P,q)$, we have that $\lvert r\rvert\geq \lvert \Sigma\rvert$.
\end{proposition}

The following result presents an upper bound for the syntactic operator and a lower bound for the semantic one.

\begin{proposition}\label{prop:boundsOps}
Let $P$ be a program over $\sign$ and $q\in \Sigma$.
Then, 
\begin{itemize}

\item[$\bullet$] $dist(P,\fSPOP(P,q))\leq (\lvert \fSem(P,q) \rvert +\lvert P \rvert)\times 2\lvert \sign \rvert$;
 
\item[$\bullet$] $dist(P,\fSem(P,q))\geq (\lvert \fSem(P,q) \rvert -\lvert P \rvert)\times \lvert \sign \rvert$.

\end{itemize}

\end{proposition}

The upper and lower bounds of the previous result depend heavily on the size of the result of forgetting for each operator. In general the result of forgetting using the semantic operator has many more rules than the result using the syntactic one.

\begin{proposition}\label{prop:sizesComparing}
Let $P$ be a program over $\sign$ and $q\in \Sigma$.
For each rule $r\in \fSPOP(P,q)$ there are at least $2^{D}$ rules in $\fSem(P,q)$, with $D=Min(\lvert H(r)\rvert, \lvert \Sigma\setm\Sigma(r)\rvert)$. 
\end{proposition}

\begin{example}
For $P=\{q\la s; q\vee u\la r; t \la q; v\la \nf q\}$, we obtain $\fSPOP(P,q)$ as follows: 
\begin{align*}
t & \la s  & t\vee u & \la r  & u & \la r, \nf s, \nf\nf u \\
v&\la \nf s, \nf\nf u & v &\la \nf s, \nf r
\end{align*}

The distance between $\fSPOP(P,q)$ and $P$ is $dist(P,\fSPOP(P,q))=16$. The result of the forgetting according to $\fSem$ has 73 rules. The distance to the original program $P$ is $dist(P,\fSem(P,q))=486$, more than thirty times the distance in the case of $\fSPOP$.
\end{example}

The result of forgetting according to $\fSPOP$ is clearly closer to the original program than that obtained by the $\fSem$ using the counter-models construction of~\cite{WangWZ13,WangZZZ14}.
Although this construction improves on that of~\cite{CabalarF07}  by doing better than just giving one rule for each counter-model, we could have considered the construction in~\cite{CabalarPV07}, which is also based on counter-models, but improves on the former by guaranteeing a minimal program as a result. However, the extra complexity to obtain a minimal program as result is not quantified~\cite{CabalarPV07}. More importantly, being also a semantic construction, the resulting program will most likely not resemble the original program at all.

\section{Related Work}
\label{sec:relWork}

Related work can be divided into two groups: one of early approaches that do not consider the notion of strong equivalence and the related HT-models for their definition, and those that do.

The first approach, due to Zhang and Foo \citeyear{ZhangF06}, belongs to the former group and consists of a pair of purely syntactic operators, strong and weak forgetting, defined for normal logic programs. 
As argued by Eiter and Wang \citeyear{EiterW08}, both lack the semantic underpinnings and are syntax-sensitive, i.e., even answer sets are not preserved while forgetting, and this necessarily generalizes to preserving HT-models (modulo the forgotten atoms).
Thus \pSP\ cannot be preserved even in very simple cases.
For example, strong forgetting about $q$ from $P=\{p\la \nf q;$ $p\la \nf p\}$ results in $\{p\la \nf p\}$, an inconsistent program, and weak forgetting about $q$ from $P=\{q\la;$ $p\la\nf q\}$ results in $\{p\la\}$, a counter-intuitive result \cite{EiterW08}.

Semantic Forgetting \cite{EiterW08} aims at preserving the answer sets while forgetting to overcome the problems of these two operators. It is defined for disjunctive programs, but, among the three forgetting operators presented, only $\forget_3$ is syntactic (the other two create canonical representations based on the computed answer sets).
While being preferable over strong and weak forgetting for its semantic foundation, it is not sufficient since it focuses on \emph{equivalence} instead of \emph{strong equivalence}. Hence, even for simple examples, \pSP\ cannot be satisfied.
Consider forgetting about $q$ from $P'=\{p\la q;$ $q\la \nf c\}$: we would expect the result to be $\{p\la \nf c\}$, since we want to preserve the dependency between $c$ and $p$, however, $\forget_3$ returns $\{p\la\}$. Both programs have the same answer sets -- hence are \emph{equivalent} -- but do not have the same logical meaning, witnessed by the fact that they are not \emph{strong equivalent}. 
In addition, while the operator is defined for disjunctive programs, its output may contain double negation, which means that the operator cannot be iterated.
To partially circumvent this, a restriction on programs is introduced under which it can be iterated \cite{EiterW08}. 
Yet, this restriction excludes very basic examples such as forgetting about $q$ from $P=\{c\la \nf p; p\la \nf q; q\la \nf p\}$.
Moreover, an initial transformation is applied in $\forget_3$ to obtain a negative normal form, i.e., a program without positive atoms in rule bodies.
While facilitating the presentation of the algorithm, it considerably reduces the similarity between the input program and its forgetting result, unnecessarily increasing the number of rules in the resulting program, even affecting rules that do not contain the atom(s) to be forgotten.

The only syntactic HT-models based operator is strong as-forgetting \cite{KnorrA14} -- all other HT-models based approaches in the literature are limited to a semantic characterization of the result, not presenting a concrete forgetting operator, or only one which is based on the notion of counter-models in HT \cite{CabalarF07}, with the drawbacks discussed in Sec.~\ref{sec:intro}.
While strong as-forgetting provides the expected results in certain cases, such as $P'$ above, its applicability is severely limited to certain programs within a non-standard class with double negation, but without disjunction. Additionally, its result often does not belong to the class of programs for which it is defined, preventing its iterative use. $\fSPOP$ overcomes all the shortcomings of strong as-forgetting and satisfies \pSP\ whenever this is possible.

\section{Conclusions}
\label{sec:conclusions}

We proposed a concrete forgetting operator for the class of extended logic programs which satisfies several important properties, notably \pSP\ whenever this is possible. Equally important is the fact that the operator is defined in a syntactic way, thus not requiring the calculation of models, and producing a program that is somehow close to the original program, hence keeping its declarative character. In particular, the rules that do not mention the atom to be forgotten remain intact. Furthermore, while the remaining rules may grow exponentially -- which is unsurprising -- even if only on the number of rules that mention the atom being forgotten, often the resulting program is the closest to the initial program, as would be the case in the example described in Sect.\ \ref{sec:intro}.
Overall, this is a substantial improvement over the existing operators, which either did not obey most desirable properties \cite{ZhangF06,EiterW08}, or were only defined for very restricted classes of programs \cite{KnorrA14}, or were based on generic semantic methods that required the computation of models and produced large programs that were syntactically very different from the original ones.
Furthermore, we characterised a class of programs -- \emph{q-forgettable} -- whose membership can be checked in linear time, and for which we can guarantee \pSP. For cases where \pSP\ cannot be ensured, our operator can nevertheless still be employed with desirable properties, such as guaranteeing that all answer sets are preserved, along the lines of one of the alternatives semantically investigated in \cite{GoncalvesKLW17}. In other words, our operator can be used without having to perform the computationally expensive semantic check ($\Omega$) for \emph{strong persistence}, ensuring the ``\emph{perfect}'' result whenever \pSP\ is possible, and a ``\emph{good}'' result when ``\emph{perfection}'' is not possible. 

Future work includes dealing with (first-order) variables, and following the other alternatives proposed in \cite{GoncalvesKLW17} to deal with cases where it is not possible to guarantee \pSP.

\paragraph{Acknowledgments}
M.\ Berthold was partially supported by the International MSc Program in Computational Logic (MCL).
R.\ Gon\c{c}alves, M.\ Knorr, and J.\ Leite were partially supported by FCT projects FORGET ({PTDC}/{CCI-INF}/{32219}/{2017}) and NOVA LINCS ({UID}/{CEC}/{04516}/{2019}).

\bibliographystyle{acmtrans}
\bibliography{bib}

\newpage
\appendix
\section{Proofs}
\begin{proposition*}{prop:normalForm}
Let $P$ be a program. 
Then, $\NF{P}$ is in normal form and is strongly equivalent to $P$.
\end{proposition*}
\begin{proof}
First we show that $\NF{P}$ is in normal form, i.e., that all conditions of Def.~\ref{def:normalForm} are satisfied.
This follows easily from the construction of $\NF{P}$: the first item of Def.~\ref{def:normalForm} is ensured by condition 1.\ and 2.\ of Def.~\ref{def:constructionNormalForm}, the second by conditions 1.\ and 3., the third by conditions 1.\ and 4., and the forth by condition 5.

We now prove that $\NF{P}$ is strongly equivalent to $P$. For that we need to guarantee that each step in the construction of $\NF{P}$ preserves strong equivalence. We rely on the well-known result that two programs are strongly equivalent iff they have the same $HT$-models (see, e.g., \cite{WangWZ13}).

\begin{itemize}
 \item Consider the case: $r$ is tautological.\\
 It is easy to see that a tautological rule is satisfied by every $HT$-interpretation. Therefore, we have that $P$ and $P\setminus\{r\}$ have the same $HT$-models.

\item Consider the case: $q\in \pbody{r}\cap\nnbody{r}$.\\
For any $HT$-interpretation $\tuple{X,Y}$, $Y$ is a classical model for $r$ if and only if it is a model for $\head{r}\la\body{r}\setm\{\nf\nf q\}$. Also the two rule are equal, when reducing them with $Y$, i.e. $r^Y=\head{r}\la\body{r}\setm\{\nf\nf q\}^Y$, which is why $X$ either models neither or both of the reduced rules. Hence $\tuple{X,Y}$ is a $HT$-model for $r$ iff it is a $HT$-model for the rule $\head{r}\la\body{r}\setm\{\nf\nf q\}$, which is why $\nf\nf q$ can be omitted in $r$.

\item Consider the case: $q\in \head{r}\cap\nbody{r}$.\\
For the rule body to be satisfied by an interpretation $\tuple{X,Y}$, necessarily $q\not\in Y$ and then $q\not\in X$. Then $q$ in the rule head is also not satisfied and therefore has no further effect on the truth of it. Therefore $q$ can be omitted in the head of $r$.

\item Consider the case: $r$ not minimal in $P$.\\
We aim to prove that $P$ and $P\setminus\{r\}$ have the same $HT$-models. If $r$ is not minimal, then there exists $r'\in P$ such that $r'\not=r$, $\head{r'}\subseteq\head{r}$ and $\body{r'}\subseteq\body{r}$. Let $I=\langle H,T\rangle$ be a $HT$-interpretation. If $I$ is a model of $P$, i.e., $I$ satisfies every rule of $P$, then trivially $I$ also satisfies $P\setminus\{r\}$. Now suppose that $I$ does not satisfy $P$. Then it does not satisfy some rule $r''\in P$. If $r''$ is not $r$ then, $I$ also does not satisfy $P\setminus\{r\}$. The interesting case is when $I$ does not satisfy $r$. This case corresponds to $I\models \body{r}$ and $I\not\models \head{r}$. But since $\body{r'}\subseteq\body{r}$ and $\head{r'}\subseteq\head{r}$ we can conclude that $I$ does not satisfy $r'$. Therefore, $I$ does not satisfy $P\setminus\{r\}$ (because $r'\in P\setminus\{r\}$).
 \end{itemize}
\end{proof}

\begin{proposition*}{prop:compNormalForm}
Let $P$ be a program. 
Then, the normal form $\NF{P}$ can be computed in PT{\scriptsize IME}.
\end{proposition*}
\begin{proof}
Steps 1 through 3 have linear time behavior, since they are operating on each rule separately. Step 4 requires all $n$ rules to be compared with each other, and thus is quadratic. 
\end{proof}

\begin{proposition*}{prop:existence}
Let $P$ be a program over signature $\Sigma$ and $q\in \Sigma$. Then $\fSP{P}{q}$ is a program over $\Sigma\setminus\{q\}$.
\end{proposition*}
\begin{proof}
This follows directly from the construction of new rules. Only parts of rules that do not contains $q$ are used as part of new rules.
\end{proof}

\begin{theorem*}{thm:FSPsingleVariable}
Let $P$ be a program over signature $\Sigma$ and $q\in \Sigma$. Then, for any $\fgt\in \FSPOP$, we have that 
$\HT(\fSP{P}{q})=\HT(\f{P}{\{q\}})$.
\end{theorem*}
\begin{proof}
To prove that \fSPOP\ behaves according to the semantics of \FSPOP, it is worth taking a look at the definition of \FSPOP. All possible answersets $Y$ of the program after forgetting are considered separately. An interpretation $\tuple{X,Y}$ is a models of $\f{P}{\{q\}}$ if there is consensus among the relevant answersets $Y\cup A$ with $A\subseteq \{q\}$ of the original program.
Since the set $\{q\}$ is unary, there are just two subsets $A$ of $\{q\}$: $\{q\}$ and $\emptyset$. Therefore for each interpretation $Y$ exactly one of these cases it true:
\begin{description}
\item[ 1] $\tuple{Y,Y} \not\in \HT(P)$ and $\tuple{Y,Y\cup \{q\}} \not\in \HT(P)$ and $\tuple{Y\cup \{q\}, Y\cup \{q\}} \not\in \HT(P)$
\item[ 2] $\tuple{Y,Y} \not\in \HT(P)$ and $\tuple{Y,Y\cup \{q\}} \not\in \HT(P)$ and $\tuple{Y\cup \{q\}, Y\cup \{q\}} \in \HT(P)$
\item[ 3] $\tuple{Y,Y} \not\in \HT(P)$ and $\tuple{Y,Y\cup \{q\}} \in \HT(P)$ and $\tuple{Y\cup \{q\}, Y\cup \{q\}} \in \HT(P)$
\item[ 4] $\tuple{Y,Y} \in \HT(P)$ and $\tuple{Y,Y\cup \{q\}} \not\in \HT(P)$ and $\tuple{Y\cup \{q\}, Y\cup \{q\}} \not\in \HT(P)$
\item[ 5] $\tuple{Y,Y} \in \HT(P)$ and $\tuple{Y,Y\cup \{q\}} \not\in \HT(P)$ and $\tuple{Y\cup \{q\}, Y\cup \{q\}} \in \HT(P)$
\item[ 6] $\tuple{Y,Y} \in \HT(P)$ and $\tuple{Y,Y\cup \{q\}} \in \HT(P)$ and $\tuple{Y\cup \{q\}, Y\cup \{q\}} \in \HT(P)$
\end{description}
To facilitate the proof, we split the interpretations of the resulting program $\f{P}{\{q\}}$ among the second element of the tuples $Y$, i.e. $\HT_Y(P) = \{\tuple{X,Y}\mid \tuple{X,Y} \in \HT(P)\}$.
For each of these $Y$, according to lemmas \ref{lem_case_1} through \ref{lem_case_6}, it holds that $\HT_Y(\fSP{P}{q})=\HT_Y(\f{P}{\{q\}})$ for some $\fgt\in \FSPOP$, from which follows that $\HT(\fSP{P}{q})=\HT(\f{P}{\{q\}})$.
The six proofs of the lemmas all follow the same pattern. If an interpretation is expected not to be a model, it is shown that the syntactic operator generates a rule, that is not satisfied by it, and respectively that all rules are satisfied by it, if the interpretation should be a model.
To do this, we always consider rules from the original program that are not or respectively are satisfied by an original interpretation. When looking for a rule that is not satisfied by an interpretation specifically, we make case distinctions over the form of the rule, namely where $q$ occurs within the body or the head. Some of these cases can then be ruled out beforehand, because for example, a interpretation $\tuple{X,Y\cup \{q\}}$ satisfies all rules that have $\nf q$ in their bodies. So, without always referring back to them, the proofs use the observations of Lem.~\ref{lem_aux_proof}.
\end{proof}

\begin{theorem*}{thm:fSPsatSP}
Let $P$ be a program over $\Sigma$ and $q\in \Sigma$. If $\tuple{P,\{q\}}$ does not satisfy $\Omega$, then $\fSPOP$ satisfies \pSP$_{\tuple{P,\{q\}}}$.
\end{theorem*}

\begin{proof}
According to Thm.~4 of \cite{GoncalvesKL-ECAI16}, every $\fgt\in\FSPOP$ satisfies $\pSP_{\tuple{P,V}}$ for every $\tuple{P,V}$ that does not satisfy $\Omega$. Hence, when applying Thm.~\ref{thm:FSPsingleVariable}, we have:
$$\as{\fSP{P}{q}\cup R}=\as{\f{P}{\{q\}}\cup R}=\as{P\cup R}_{\parallel \{q\}}$$
for all programs $R$ over $\sign\setm \{q\}$, if $\tuple{P,\{q\}}\notin \Omega$.
\end{proof}

\begin{proposition*}{prop:FSPsatSI}
Let $P$ be a program over $\Sigma$ and $q\in \Sigma$. Then,
$\fSP{P\cup R}{q}=\fSP{P}{q}\cup R$, for all programs $R$ over $\Sigma\setm\{q\}$.
\end{proposition*}

\begin{proof}
This follows directly from the construction of new rules. Rules that do not contain $q$ are directly added to the result, without any modifications. 
\end{proof}

\begin{theorem*}{prop:FSPsatwSP}
Let $P$ be a program over $\Sigma$ and $q\in \Sigma$. Then,
$\as{P\cup R}_{\parallel \{q\}}\subseteq \as{\fSP{P}{q}\cup R}$, for all programs $R$ over $\sign\setm \{q\}$.
\end{theorem*}
\begin{proof}
According to Thm.~3 of \cite{GoncalvesKLW17} the class \FSPOP\ satisfies the the property \pwSP, which means that for each $\fgt\in \FSPOP$, program $P$ and $V\subseteq \sign$, we have ${\as{P\cup R}_{\parallel V}}\subseteq{\as{\f{P}{V}\cup R}}$, for all programs $R$ over $\sign\setm V$. Applying Thm.~\ref{thm:FSPsingleVariable} we get:
$$\as{P\cup R}_{\parallel \{q\}}\subseteq\as{\f{P}{\{q\}}\cup R}=\as{\fSP{P}{q}\cup R}$$
\end{proof}

\begin{proposition*}{prop:FSPsatSE}
Let $P$ and $P'$ be programs over $\Sigma$ and $q\in \Sigma$. If $P\equiv P'$ then $\fSP{P}{q}\equiv \fSP{P'}{q}$.
\end{proposition*}
\begin{proof}
According to Prop.~1 of \cite{GoncalvesKLW17} the class \FSPOP\ satisfies the the property \pSE, which means that for each $\fgt\in \FSPOP$, extended programs $P$ and $P'$, and $V\subseteq \sign$: if $P\equiv P'$, then
$\f{P}{V}\equiv\f{P'}{V}$.
Applying Thm.~\ref{thm:FSPsingleVariable} twice we get:
$$\fSP{P}{q}\equiv\f{P}{\{q\}}\equiv\f{P'}{\{q\}}\equiv\fSP{P'}{q}$$
\end{proof}

\begin{theorem*}{thm:qforgetable_implies_sp}
Let $P$ be a program over $\Sigma$, and $q\in \Sigma$. If $P$ is $q$-forgettable, then $\langle P,\{q\}\rangle$ does not satisfy $\Omega$.
\end{theorem*}
\begin{proof}
In the following we show,
\begin{itemize}
    \item[1)] If $\Omega$ holds for $\langle P,\{q\}\rangle$, there must be a self-cycle for $q$, i.e. $R_3\not=\emptyset$,
    \item[2)] Having $q$ only appear in self-sycles, implies that $\Omega$ does not hold, and
    \item[3)] Having the fact $q\la$ in $P$ implies that $\Omega$ does not hold,
\end{itemize}
Which suffices as proof of Thm.~\ref{thm:qforgetable_implies_sp}.\\
1)\\
We recall the $\Omega$-criterion:
Let $P$ be a program over $\sign$ and $V\subseteq \sign$. 
An instance $\tuple{P,V}$ \emph{satisfies criterion $\Omega$} if there exists $Y\subseteq \sign\text{\textbackslash} V$ such that the set of sets
\[ \SRel^Y=\{\RA_{\tuple{P,V}}^{Y,A}\mid A\in \Rel_{\tuple{P,V}}^Y\}\]
 is non-empty and has no least element, where 
\begin{align*}
\RA^{Y,A}_{\tuple{P,V}} & =\{X\text{\textbackslash} V\mid \tuple{X,Y\cup A}\in \HT(P)\} \\ 
\Rel_{\tuple{P,V}}^Y & =\{A\subseteq V\mid \tuple{Y\cup A,Y\cup A}\in \HT(P) \text{ and } \\
& \hspace{0.5cm} \nexists A'\subset A \text{ s.t.\ }\tuple{Y\cup A',Y\cup A}\in \HT(P)\}.
\end{align*}
For $\mathcal{R}_{\tuple{P,\{q\}}}^Y$ to be non-empty and have no least element, surely it has to have at least two elements. Then $\RA^{Y,A}_{\tuple{P,\{q\}}}$ must have at least two elements as well, which is only the case if
\begin{itemize}
    \item $\langle Y,Y\rangle\models P$,
    \item $\langle Y,Y\cup\{q\}\rangle\not\models P$, and
    \item $\langle Y\cup\{q\},Y\cup\{q\}\rangle\models P$.
\end{itemize}
We use the observations of Lem.~\ref{lem_aux_proof} to make a pigeonhole argument that $R_3\not=\emptyset$:\\
$\langle Y,Y\cup\{q\}\rangle$ can only be contradicted by $r\in R$, $r_2\in R_2$, $r_3\in R_3$ and $r_4\in R_4$. Any $r\in R\cup R_2$ that contradicts $\langle Y,Y\cup\{q\}\rangle$ also contradicts $\langle Y\cup\{q\},Y\cup\{q\}\rangle$, and any $r_4\in R_4$ that contradicts $\langle Y,Y\cup\{q\}\rangle$ also contradicts $\langle Y,Y\rangle$. Therefore for a forgetting instance $\langle P,\{q\}\rangle$ to satisfy $\Omega$, P must indeed contain a self-cycle on $q$.

2)\\
If there are no rules $r\in R_0\cup R_1\cup R_2\cup R_4$, for any $Y\in\sign(P)\setminus \{q\}$ the reduced programs $P^{Y}$ and $P^{Y\cup\{q\}}$ must have the same classical models. Then $\mathcal{R}_{\tuple{P,\{q\}}}^Y$ cannot have more than one element. Then $\langle P,\{q\}\rangle$ does not satisfy $\Omega$.

3)\\
Like in (1), for a forgetting instance $\langle P,\{q\}\rangle$ to satisfy $\Omega$, there must be a $Y\in\sign(P)\setminus \{q\}$, such that:
\begin{itemize}
    \item $\langle Y,Y\rangle\models P$,
    \item $\langle Y,Y\cup\{q\}\rangle\not\models P$, and
    \item $\langle Y\cup\{q\},Y\cup\{q\}\rangle\models P$.
\end{itemize}
If $q\la\in P$, then $\langle Y,Y\rangle\not\models P$ for all $Y\in\sign(P)\setminus \{q\}$. Then $\langle P,\{q\}\rangle$ does not satisfy $\Omega$.
\end{proof}

In order to introduce the concrete semantical operator, $\fSem$, we recall some necessary notions related to countermodels in here-and-there \cite{CabalarF07}, which have been used previously in a similar manner for computing concrete results of forgetting for  classes of forgetting operators based on HT-models \cite{WangWZ13,WangZZZ14}.

Essentially, the HT-interpretations that are not HT-models of $P$ (hence the name countermodels) can be used to determine rules, that, if conjoined, result in a program $P'$ that is strongly equivalent to $P$. More precisely, an HT-interpretation $\tuple{X,Y}$ is an \emph{HT-countermodel} of $P$ if $\tuple{X,Y}\not \models P$.
We also define the following rules:
\begin{align}
\cmr{X,Y}{} & = (Y\setm X) \la X, \nf (\sign\setm Y), \nf\nf (Y\setm X)\\
\cmr{Y,Y}{} & = \emptyset \la Y, \nf (\sign\setm Y)
\end{align}

The operator $\fSem$ can then be defined for each program $P$ and set of atoms $V$ as 
\begin{align*}
\fSem(P,V)  = &\{\cmr{X,Y}{}\mid \tuple{X,Y}\notin \classMFSP \text{ and } \tuple{Y,Y}\in \classMFSP\}\\
& \cup \{\cmr{Y,Y}{}\mid \tuple{Y,Y}\notin \classMFSP\}.
\end{align*}
Please note that in the case of forgetting one atom, we abuse notation and write $\fSem(P,q)$ instead of $\fSem(P,\{q\})$. \\
\begin{proposition*}{prop:checking_q-forgettable_complexity}
Let $P$ be a program over $\Sigma$, and $q\in \Sigma$. Deciding if $P$ is $q$-forgettable can be done in linear time.
\end{proposition*}
\begin{proof}
The restrictions posed for $q$-forgettable programs take into account each rule of the program separately. Therefore checking for membership of the class only requires looking at each rule contained in a program once.
\end{proof}

\begin{theorem*}{def:OurForgettingSimple}
Let $P$ be a program over $\Sigma$, and $q\in \Sigma$.
Then, $\fSP{P}{q}$ is constructed using only the derivation rules 1a, 1b and 4.
\end{theorem*}
\begin{proof}
\begin{itemize}
    \item If all occurrences of $q$ are within self-cycles, then $R_0\cup R_1\cup R_2\cup R_4=\emptyset$. Then none of the derivation rules is used and only rules not containing $q$ are transferred to $\fgt_{SP}(P,q)$.
    \item If the fact $q\la$ is contained in $P$. Then $\asdual{q}{R_4}=\asdual{q}{R_1\cup R_4}=\emptyset$, which is why derivation rules 5 and 6 do not produce any rules. Furthermore, derivation rule 1 produces rules that are always more minimal than rules derived by 2, 3 and 7, which is why these derivation rule can be disregarded as well.
    \item If $P$ contains no self-cycle for $q$, then derivation rules 2, 3, 5, 6 and 7 do not apply.
\end{itemize}
\end{proof}

\begin{theorem*}{thm:complexity}
Let $P$ be a program over $\Sigma$ and $p\in\Sigma$. Then, computing $\fSP{P}{q}$ is in EXPT{\scriptsize IME} in the number of rules containing occurrences of $q$ and linear in the remaining rules.
\end{theorem*}
\begin{proof}
This follows directly from the construction of new rules. When used, each of the derivation rules has most has an exponential running time. Each of them can at most be used a polynomial number of times.
\end{proof}

\begin{proposition*}{prop:sizeOfRulesFSem}
Let $P$ be a program over $\sign$ and $q\in \Sigma$.
Then, for each $r\in \fSem(P,q)$, we have that $\lvert r\rvert\geq \lvert \Sigma\rvert$.
\end{proposition*}
\begin{proof}
The result follows easily from the observation that the rules generated using the counter-models construction have at least all atoms of the signature. 
\end{proof}

\begin{proposition*}{prop:boundsOps}
Let $P$ be a program over $\sign$ and $q\in \Sigma$.
Then, 
\begin{itemize}
 
\item[$\bullet$] $dist(P,\fSem(P,q))\geq (\lvert \fSem(P,q) \rvert -\lvert P \rvert)\times \lvert \sign \rvert$;

\item[$\bullet$] $dist(P,\fSPOP(P,q))\leq (\lvert \fSem(P,q) \rvert +\lvert P \rvert)\times 2\lvert \sign \rvert$.

\end{itemize}

\end{proposition*}

\begin{proof}
The upper bound for $\fSPOP(P,q)$ follows from the observation that an upper bound for the distance of two programs is the sum of the sizes of each rule of the programs. In this case, assuming that both $P$ and $\fSPOP(P,q)$ are in the normal form, the size of a rule is limited by 2$\lvert \sign \rvert$, since an atom may appear at most twice in a rule.

The lower bound $\fSem(P,q)$ is obtained by considering a limit case of a mapping between $P$ and $\fSem(P,q)$ such that the distance between the rules associated with the mapping is 0. In this case we would be left with the size of the remaining $\lvert \fSem(P,q) \rvert -\lvert P \rvert$ rules of $\fSem(P,q)$, each of which, according to Prop.~\ref{prop:sizeOfRulesFSem}, has a size of at least $\lvert \Sigma\rvert$.
\end{proof}

\begin{proposition*}{prop:sizesComparing}
Let $P$ be a program over $\sign$ and $q\in \Sigma$.
For each rule $r\in \fSPOP(P,q)$ there are at least $2^{D}$ rules in $\fSem(P,q)$, with $D=Min(\lvert H(r)\rvert, \lvert \Sigma\setm\Sigma(r)\rvert)$. 
\end{proposition*}
\begin{proof}
First note that from Thm.~\ref{thm:FSPsingleVariable} we have that $\HT(\fSPOP(P,q))=\HT(\fSem(P,q))$, and therefore the counter-models of $\fSPOP(P,q)$ are exactly the counter-models used to construct the rules of $\fSem(P,q)$.

Now let $r\in \fSPOP(P,q)$ such that $r$ in non-tautological and let $S\subseteq \head{r}\cup \Sigma\setm\Sigma(r)$.
First suppose first that $\tuple{\pbody{r}\cup \nnbody{r}\cup S,\pbody{r}\cup \nnbody{r}\cup S}$ is an HT-model of $\fSPOP(P,q)$. In this case, since $\tuple{\pbody{r},\pbody{r}\cup \nnbody{r}\cup S}$ is always a counter-model of $r$, the following rule is in $\fSem(P,q)$:  
\[(\nnbody{r}\cup S)\la \pbody{r}, \nf (\sign\setm \sign(r)), \nf \nf (\nnbody{r}\cup S).\]

If $\tuple{\pbody{r}\cup \nnbody{r}\cup S,\pbody{r}\cup \nnbody{r}\cup S}$ is not an HT-model of $\fSPOP(P,q)$, then by construction, the following rule is in $\fSem(P,q)$:
\[\la \pbody{r}, \nnbody{r}, S, \nf (\Sigma) .\]
\end{proof}

\begin{lemma}\label{lem_aux_proof}
Let $P$ be a program in normal form over $\Sigma$, $q\in \Sigma$, $X\subseteq Y\subseteq \sign\setminus \{q\}$, and $R$, $R_0$, $R_1$, $R_2$, $R_3$, $R_4$ following subsets of $P$:
\begin{align*}
R_\text{ }&:= \{r\in P\mid q\not\in \Sigma(r)\} 
&
R_2&:= \{r\in P\mid not\ not\ q \in B(r), q\not\in H(r)\}
\\
R_0&:= \{r\in P\mid q\in B(r)\}
&
R_3&:= \{r\in P\mid not\ not\ q \in B(r), q\in H(r)\}
\\
R_1&:= \{r\in P\mid not\ q\in B(r)\}
&
R_4&:= \{r\in P\mid not\ not\ q \not\in B(r), q\in H(r)\}
\end{align*}
Then,
\begin{align*}
\tuple{X,Y}&\models R_0\cup R_2\cup R_3\\
\tuple{X,Y\cup \{q\}}&\models R_0\cup R_1\\
\tuple{X\cup \{q\},Y\cup \{q\}}&\models R_1\cup R_3\cup R_4
\end{align*}
Therefore,
\begin{align*}
\tuple{X,Y}\not\models P&\leftrightarrow\exists r\in R\cup R_1\cup R_4:\tuple{X,Y}\not\models r\\
\tuple{X,Y\cup \{q\}}\not\models P&\leftrightarrow\exists r\in R\cup R_2\cup R_3\cup R_4:\tuple{X,Y\cup \{q\}}\not\models r\\
\tuple{X\cup \{q\},Y\cup \{q\}}\not\models P&\leftrightarrow\exists r\in R\cup R_0\cup R_2:\tuple{X\cup \{q\},Y\cup \{q\}}\not\models r
\end{align*}
Also, given that $Y\models P$ and $Y\cup \{q\}\models P$,
$$\forall r_4\in R_4:\tuple{X,Y}\models r_4\leftrightarrow \tuple{X,Y\cup \{q\}}\models r_4$$ 
and
$$\forall r_2\in R_2:\tuple{X,Y\cup \{q\}}\models r_2\leftrightarrow \tuple{X\cup \{q\},Y\cup \{q\}}\models r_2$$
\end{lemma}
\begin{proof}
The proof of this result follows from the following simple observations:
\begin{itemize}
    \item $\tuple{X,Y}\models R_0\cup R_2\cup R_3$, because the interpretation neither satisfies $q$, nor $\nf\nf q$.
    \item $\tuple{X,Y\cup \{q\}}\models R_0\cup R_1$, because the interpretation neither satisfies $q$, nor $\nf q$.
    \item $\tuple{X\cup \{q\},Y\cup \{q\}}\models R_1\cup R_3\cup R_4$, because the interpretation does not satisfy $\nf q$ in the rule bodies of $R_1$, and satisfies $q$ in the rule heads of $R_3\cup R_4$
\end{itemize}
Since $\langle R,R_0,R_1,R_2,R_3,R_4\rangle$ is a partition of $P$, and we know which types of rules cannot contradict the respective interpretations, if an interpretation is no model for the program $P$, it must be contradicted by one of the remaining types of rules.

If $Y\models P$ and $Y\cup \{q\}\models P$, then $\forall r_4\in R_4:\tuple{X,Y}\models r_4\leftrightarrow \tuple{X,Y\cup \{q\}}\models r_4$, because rules of $R_4$ are not dependent on whether $q\in Y$.

$\forall r_2\in R_2:\tuple{X,Y\cup \{q\}}\models r_2\leftrightarrow \tuple{X\cup \{q\},Y\cup \{q\}}\models r_2$, because rules of $R_2$ are not dependent on whether $q\in X$.
\end{proof}

\begin{lemma}\label{lem_case_1}
Let $P$ be a program over $\Sigma$, $q\in \Sigma$, $Y\subseteq \sign\setminus \{q\}$ and $\tuple{Y,Y} \not\in \HT(P)$ and $\tuple{Y,Y\cup \{q\}} \not\in \HT(P)$ and $\tuple{Y\cup \{q\}, Y\cup \{q\}}\not\in \HT(P)$.\\
Then $\HT_Y(\fSP{P}{q})=\HT_Y(\f{P}{\{q\}})$ for a representative $\fgt\in\FSPOP$.
\end{lemma}
\begin{proof}
In this setting, neither $\{q\}$, nor $\emptyset$ are relevant, i.e. $Rel^Y_{(P,\{q\})} = \emptyset$, and therefore\\
$\HT_Y(\f{P}{\{q\}}) = \emptyset$.\\
From the premises it follows directly that $Y\not\models P$ and $Y\cup \{q\}\not\models P$. The former implies that $\exists r\in R_1\cup R_4:Y\not\models r$. The latter that $\exists r\in R_0\cup R_2:Y\cup \{q\}\not\models r$. Then there are two cases:
\begin{itemize}
    \item $\exists r\in R_0\cup R_2:Y\cup \{q\}\not\models r$ and $\exists r_4\in R_4:Y\not\models r_4$
    \item $\exists r\in R_0\cup R_2:Y\cup \{q\}\not\models r$, $\not\exists r_4\in R_4:Y\not\models r_4$ and $\exists r_1\in R_1:Y\not\models r_1$
\end{itemize}
In the first case, by derivation rule 1a or 1b we construct a new rule\\
$\head{r}\cup \headwoq{r_4}\la \bodywoq{r}\cup \body{r_4}$ or $\head{r}\la \bodywoq{r}\cup \nf\nf(\body{r_4})\cup \nf(\headwoq{r_4})$ respectively, for both of which we have that their heads are false for $Y$ and their bodies are true. Hence $\HT_Y(\fSP{P}{q})=\emptyset$
In the second case, we make the additional assumption that $\not\exists r_3\in R_3:Y\not\models r_3$. Then $\exists D\in \asdual{q}{R_3\cup R_4}:Y\models D$. Therefore $Y\not\models \head{r_1}\la\bodywoq{r_1}\cup D$ which is constructed by 4. Hence $\HT_Y(\fSP{P}{q})=\emptyset$.
If in turn $\exists r_3\in R_3:Y\not\models r_3$, then\\
$Y\not\models \head{r_1}\la\bodywoq{r_1}\cup \nf\nf(\bodywoq{r_3})\cup \nf(\headwoq{r_3})\cup \nf\nf(\bodywoq{r})\cup \nf(\body{r})\cup D$ which is constructed by rule 5. But then also $\HT_Y(\fSP{P}{q})=\emptyset$.
In any case we have $\HT_Y(\fSP{P}{q})=\emptyset$ and therefore $\HT_Y(\fSP{P}{q})=\HT_Y(\f{P}{\{q\}})$.
\end{proof}

\begin{lemma}\label{lem_case_2}
Let $P$ be a program over $\Sigma$, $q\in \Sigma$, $Y\subseteq \sign\setminus \{q\}$ and $\tuple{Y,Y} \not\in \HT(P)$ and $\tuple{Y,Y\cup \{q\}} \not\in \HT(P)$ and $\tuple{Y\cup \{q\},Y\cup \{q\}} \in \HT(P)$.\\
Then $\HT_Y(\fSP{P}{q})=\HT_Y(\f{P}{\{q\}})$ for a representative $\fgt\in\FSPOP$.
\end{lemma}
\begin{proof}
In this setting, just $\{q\}$ is relevant, i.e. $Rel^Y_{(P,\{q\})} = \{\{q\}\}$, and therefore\\
$\HT_Y(\f{P}{\{q\}}) =\{\tuple{X\setminus \{q\},Y}\mid \tuple{X,Y\cup \{q\}} \in \HT(P)\}$.\\
First, we proof that $\HT_Y(\fSP{P}{q})\subseteq \HT_Y(\f{P}{\{q\}})$. So we suppose that $\tuple{X,Y}\not\in \HT_Y(\f{P}{\{q\}})$, then $\tuple{X\cup \{q\},Y\cup \{q\}}\not\in \HT(P)$ and $\tuple{X\setminus \{q\},Y\cup \{q\}}\not\in \HT(P)$.\\
From $\tuple{X\cup \{q\},Y\cup \{q\}}\not\in \HT(P)$ it follows that
\begin{itemize}
    \item $\exists r_2\in R_2:\tuple{X\cup \{q\},Y\cup \{q\}}\not\models r_2$, or
    \item $\exists r_0\in R_0:\tuple{X\cup \{q\},Y\cup \{q\}}\not\models r_0$
\end{itemize}
From $\tuple{X\setminus \{q\},Y\cup \{q\}}\not\in \HT(P)$ it follows that
\begin{itemize}
    \item $\exists r_2\in R_2:\tuple{X\setminus \{q\},Y\cup \{q\}}\not\models r_2$, or
    \item $\exists r_3\in R_3:\tuple{X\setminus \{q\},Y\cup \{q\}}\not\models r_3$, or
    \item $\exists r_4\in R_4:\tuple{X\setminus \{q\},Y\cup \{q\}}\not\models r_4$
\end{itemize}
From the premise of the lemma it follows that
\begin{itemize}
    \item $\exists r_4\in R_4:\tuple{Y,Y}\not\models r_4\wedge \tuple{Y,Y\cup \{q\}}\not\models r4$, or
    \item $\exists r_1\in R_1:\tuple{Y,Y}\not\models r_1\wedge \exists r_3\in R_3:\tuple{Y,Y\cup \{q\}}\not\models r3$
\end{itemize}
Going through the cases \dots
\begin{itemize}
    \item If $\exists r_2\in R_2:\tuple{X\cup \{q\},Y\cup \{q\}}\not\models r_2$ and\\
    $\exists r_4\in R_4:\tuple{Y,Y}\not\models r_4\wedge \tuple{Y,Y\cup \{q\}}\not\models r4$, then\\
    $\tuple{X,Y}\not\models \head{r_2}\cup \leftarrow \bodywoq{r_2}\cup \nf(\headwoq{r_4})\cup \nf\nf(\body{r_4})$, which is derived by 1.
    \item If $\exists r_2\in R_2:\tuple{X\cup \{q\},Y\cup \{q\}}\not\models r_2$ and\\
    $\exists r_1\in R_1:\tuple{Y,Y}\not\models r_1\wedge \exists r_3\in R_3:\tuple{Y,Y\cup \{q\}}\not\models r3$, then\\
    $\tuple{X,Y}\not\models \head{r_2}\cup \leftarrow \bodywoq{r_2}\cup \nf(\headwoq{r_3})\cup \nf\nf(\bodywoq{r_3})\cup \nf(\head{r_1})\cup \nf\nf(\bodywoq{r_1})$, which is derived by 2.
    \item If $\exists r_0\in R_0:\tuple{X\cup \{q\},Y\cup \{q\}}\not\models r_0$ and\\
    $\exists r_3\in R_3:\tuple{X\setminus \{q\},Y\cup \{q\}}\not\models r_3$ and\\
    $\exists r_4\in R_4:\tuple{Y,Y}\not\models r_4\wedge \tuple{Y,Y\cup \{q\}}\not\models r4$, then\\
    $\tuple{X,Y}\not\models\head{r_0}\cup \headwoq{r_3} \leftarrow \bodywoq{r_0}\cup \bodywoq{r_3}\cup \nf(\head{r_4})\cup \nf\nf(\bodywoq{r_4})$, which is derived by 2.
    \item If $\exists r_0\in R_0:\tuple{X\cup \{q\},Y\cup \{q\}}\not\models r_0$ and\\
    $\exists r_3\in R_3:\tuple{X\setminus \{q\},Y\cup \{q\}}\not\models r_3$ and\\
    $\exists r_1\in R_1:\tuple{Y,Y}\not\models r_1\wedge \exists r_3'\in R_3:\tuple{Y,Y\cup \{q\}}\not\models r_3'$, then\\
    $\tuple{X,Y}\not\models\head{r_0}\cup \headwoq{r_3} \leftarrow \bodywoq{r_0}\cup \bodywoq{r_3}\cup \nf(\head{r_1})\cup \nf\nf(\bodywoq{r_1})$, which is derived by 2.
    \item If $\exists r_0\in R_0:\tuple{X\cup \{q\},Y\cup \{q\}}\not\models r_0$ and\\
    $\exists r_4\in R_4:\tuple{X\setminus \{q\},Y\cup \{q\}}\not\models r_4$, then\\
    $\tuple{X,Y}\not\models\head{r_0}\cup \headwoq{r_4} \leftarrow \bodywoq{r_0}\cup \body{r_4}$, which is derived by 1.
\end{itemize}
Either way there exists a rule in $\fSP{P}{q}$ that is not satisfied by $\tuple{X,Y}$, hence $\tuple{X,Y}\not\in \HT(\fSP{P}{q})$, therefore $\HT_Y(\fSP{P}{q})\subseteq \HT_Y(\f{P}{\{q\}})$.\\
Second, we proof that $\HT_Y(\fSP{P}{q})\supseteq \HT_Y(\f{P}{\{q\}})$. So we suppose that $\tuple{X,Y}\in \HT_Y(\f{P}{\{q\}})$, then $\tuple{X\cup \{q\},Y\cup \{q\}}\in \HT(P)$ or $\tuple{X\setminus \{q\},Y\cup \{q\}}\in \HT(P)$.\\
From the premise of the lemma it follows that
\begin{itemize}
    \item $\exists r_4\in R_4:\tuple{Y,Y}\not\models r_4\wedge \tuple{Y,Y\cup \{q\}}\not\models r4$, or
    \item $\exists r_1\in R_1:\tuple{Y,Y}\not\models r_1\wedge \exists r_3\in R_3:\tuple{Y,Y\cup \{q\}}\not\models r3$
\end{itemize}
Therefore $\not\exists D\in\asdual{q}{R_3\cup R_4}:\tuple{Y,Y}\models D$ and $\not\exists D\in\asdual{q}{R_1\cup R_4}:\tuple{Y,Y}\models D$, which means that $\tuple{X,Y}$ satisfies all rules derived from 4 and 6.\\
Suppose that $\tuple{X\cup \{q\},Y\cup \{q\}}\in \HT(P)$
\begin{itemize}
    \item $\forall r\in R_0\cup R_2:\tuple{X,Y}\models \head{r}\vee \tuple{X,Y}\not\models \bodywoq{r}$. Therefore $\tuple{X,Y}$ is a model of all rules derived from 1, 2, 3, 5, 7.
\end{itemize}
Suppose that $\tuple{X\setminus \{q\},Y\cup \{q\}}\in \HT(P)$
\begin{itemize}
    \item $\forall r_3\in R_3:\tuple{X,Y}\models \head{r_3}\vee \tuple{X,Y}\not\models \bodywoq{r_3}$. Therefore $\tuple{X,Y}$ is a model of all rules derived from 2, 3, 5, 6, 7.
    \item $\forall r_3\in R_4:\tuple{X,Y}\models \head{r_4}\vee \tuple{X,Y}\not\models \bodywoq{r_4}$. Therefore $\tuple{X,Y}$ is a model of all rules derived from 1.
\end{itemize}
In either case $\tuple{X,Y}\models \fSP{P}{q}$, therefore $\HT_Y(\fSP{P}{q})\supseteq \HT_Y(\f{P}{\{q\}})$, and in the end $\HT_Y(\fSP{P}{q})=\HT_Y(\f{P}{\{q\}})$.
\end{proof}

\begin{lemma}\label{lem_case_3}
Let $P$ be a program over $\Sigma$, $q\in \Sigma$, $Y\subseteq \sign\setminus \{q\}$ and $\tuple{Y,Y} \not\in \HT(P)$ and $\tuple{Y,Y\cup \{q\}} \in \HT(P)$ and $\tuple{Y\cup \{q\},Y\cup \{q\}} \in \HT(P)$.\\
Then $\HT_Y(\fSP{P}{q})=\HT_Y(\f{P}{\{q\}})$ for a representative $\fgt\in\FSPOP$.
\end{lemma}
\begin{proof}
In this setting, neither $\{q\}$, nor $\emptyset$ are relevant, i.e. $Rel^Y_{(P,\{q\})} = \emptyset$, and therefore\\
$\HT_Y(\f{P}{\{q\}}) = \emptyset$.\\
From $\tuple{Y,Y} \not\in \HT(P)$ and $\tuple{Y,Y\cup \{q\}} \in \HT(P)$ it follows that $\exists r'\in R_1:\tuple{Y,Y}\models \bodywoq{r'}\wedge \tuple{Y,Y}\not\models \head{r'}$
From $\tuple{Y,Y\cup \{q\}}\models P$ it follows that $\forall r\in R_3\cup R_4:\tuple{Y,Y}\not\models \bodywoq{r}\vee\tuple{Y,Y}\models\headwoq{r}$ and therefore that $\exists D\in \asdual{q}{R_3\cup R_4}:\tuple{Y,Y}\models D$. Derivation 4 constructs $H(r') \leftarrow B^{\setminus q}(r')\cup D$ which is not satisfied by $\tuple{Y,Y}$. Hence $\HT_Y(\fSP{P}{q})=\emptyset$ and therefore $\HT_Y(\fSP{P}{q})=\HT_Y(\f{P}{\{q\}})$.
\end{proof}

\begin{lemma}\label{lem_case_4}
Let $P$ be a program over $\Sigma$, $q\in \Sigma$, $Y\subseteq \sign\setminus \{q\}$ and $\tuple{Y,Y} \in \HT(P)$ and $\tuple{Y,Y\cup \{q\}} \not\in \HT(P)$ and $\tuple{Y\cup \{q\},Y\cup \{q\}} \not\in \HT(P)$.\\
Then $\HT_Y(\fSP{P}{q})=\HT_Y(\f{P}{\{q\}})$ for a representative $\fgt\in\FSPOP$.
\end{lemma}
\begin{proof}
In this setting, just $\emptyset$ is relevant, i.e. $Rel^Y_{(P,\{q\})} = \{\emptyset\}$, and therefore\\
$\HT_Y(\f{P}{\{q\}})=\HT_Y(P)$.\\
First, we proof that $\HT_Y(\fSP{P}{q})\subseteq \HT_Y(\f{P}{\{q\}})$. So we suppose that $\tuple{X,Y}\not\in \HT_Y(\f{P}{\{q\}})$, then $\tuple{X,Y}\not\in \HT(P)$.\\
From $\tuple{X,Y}\not\in \HT(P)$ it follows that
\begin{itemize}
    \item $\exists r_1\in R_1: \tuple{X,Y}\not\models r_1$, or
    \item $\exists r_4\in R_4: \tuple{X,Y}\not\models r_4$
\end{itemize}
From the premise of the lemma it follows that
\begin{itemize}
    \item $\exists r_2\in R_2:\tuple{X\setminus \{q\},Y\cup \{q\}}\not\models r_2\wedge\tuple{X\cup \{q\},Y\cup \{q\}}\not\models r_2$, or
    \item $\exists r_0\in R_0:\tuple{X\setminus \{q\},Y\cup \{q\}}\not\models r_0$ and $\exists r_3\in R_3:\tuple{X\cup \{q\},Y\cup \{q\}}\not\models r_3$
\end{itemize}
Going through the cases:
\begin{itemize}
    \item In case that $\exists r_1\in R_1: \tuple{X,Y}\not\models r_1$, and $\exists r_2\in R_2:\tuple{X\setminus \{q\},Y\cup \{q\}}\not\models r_2\wedge\tuple{X\cup \{q\},Y\cup \{q\}}\not\models r_2$, we make the additional assumption that $\exists r_3\in R_3:\tuple{X,Y}\not\models r_3$
    from the premise $\tuple{Y,Y} \in \HT(P)$, we get that $\exists D\in\asdual{q}{R_4}:\tuple{X,Y}\models D$, then $\tuple{X,Y}\not\models \headwoq{r_1} \leftarrow \bodywoq{r_1}\cup \nf(\head{r_2}\cup \headwoq{r_3})\cup \nf\nf(\bodywoq{r_2}\cup \bodywoq{r_3})\cup D$, which is derived by 5.
    \item If in turn $\exists r_1\in R_1: \tuple{X,Y}\not\models r_1$, and $\exists r_2\in R_2:\tuple{X\setminus \{q\},Y\cup \{q\}}\not\models r_2\wedge\tuple{X\cup \{q\},Y\cup \{q\}}\not\models r_2$, but $\not\exists r_3\in R_3:\tuple{X,Y}\not\models r_3$, then we have that $\exists D\in\asdual{q}{R_3\cup R_4}:\tuple{X,Y}\models D$ and therefore\\
    $\tuple{X,Y}\not\models \headwoq{r_1} \leftarrow \bodywoq{r_1}\cup D$, which is derived by 4.
    \item If $\exists r_1\in R_1: \tuple{X,Y}\not\models r_1$, and $\exists r_0\in R_0:\tuple{X\setminus \{q\},Y\cup \{q\}}\not\models r_0$ and $\exists r_3\in R_3:\tuple{X\cup \{q\},Y\cup \{q\}}\not\models r_3$
    from the premise $\tuple{Y,Y} \in \HT(P)$, we get that $\exists D\in\asdual{q}{R_4}:\tuple{X,Y}\models D$, then $\tuple{X,Y}\not\models \headwoq{r_1} \leftarrow \bodywoq{r_1}\cup \nf(\head{r_0}\cup \headwoq{r_3})\cup \nf\nf(\bodywoq{r_0}\cup \bodywoq{r_3})\cup D$, which is derived by 5.
    \item All cases with $\exists r_4\in R_4: \tuple{X,Y}\not\models r_4$ are absolutely analogous to the three cases above.
\end{itemize}
Second, we proof that $\HT_Y(\fSP{P}{q})\supseteq \HT_Y(\f{P}{\{q\}})$. So we suppose that $\tuple{X,Y}\in \HT_Y(\f{P}{\{q\}})$, therefore $\tuple{X,Y}\in \HT(P)$.\\
From $\tuple{X,Y}\in \HT(P)$ it follows that $\forall r'\in R_1\cup R_4:\tuple{X,Y}\models\headwoq{r'}\vee\tuple{X,Y}\not\models\bodywoq{r'}$, therefore $\tuple{X,Y}$ satisfies all rules derived by 1, 2, 4, 5, 6.\\
From $\tuple{Y\cup \{q\},Y\cup \{q\}}\not\models P$ it follows that $\exists r\in R_0\cup R_2:\tuple{Y,Y}\not\models r$ and therefore $\forall D\in \asdual{q}{R_0\cup R_2}:\tuple{Y,Y}\not\models D$, which means that $\tuple{X,Y}$ satisfies all rules derived by 3 and 7.
So $\tuple{X,Y}\models \fSP{P}{q}$, therefore $\HT_Y(\fSP{P}{q})\supseteq \HT_Y(\f{P}{\{q\}})$, and in the end $\HT_Y(\fSP{P}{q})=\HT_Y(\f{P}{\{q\}})$.
\end{proof}

\begin{lemma}\label{lem_case_5}
Let $P$ be a program over $\Sigma$, $q\in \Sigma$, $Y\subseteq \sign\setminus \{q\}$ and $\tuple{Y,Y} \in \HT(P)$ and $\tuple{Y,Y\cup \{q\}} \not\in \HT(P)$ and $\tuple{Y\cup \{q\},Y\cup \{q\}} \in \HT(P)$.\\
Then $\HT_Y(\fSP{P}{q})=\HT_Y(\f{P}{\{q\}})$ for a representative $\fgt\in\FSPOP$.
\end{lemma}
\begin{proof}
In this setting, both $\emptyset$ and $\{q\}$ are relevant, i.e. $Rel^Y_{(P,\{q\})} = \{\emptyset, \{q\}\}$, and therefore\\
$\HT_Y(\f{P}{\{q\}}) =\{\tuple{X\setminus \{q\},Y}\mid \tuple{X,Y} \in \HT(P) \wedge \tuple{X,Y\cup \{q\}} \in \HT(P)\}$.\\
First, we proof that $\HT_Y(\fSP{P}{q})\subseteq \HT_Y(\f{P}{\{q\}})$. So we suppose that $\tuple{X,Y}\not\in \HT_Y(\f{P}{\{q\}})$, then $\tuple{X,Y} \not\in \HT(P)$ or $\tuple{X,Y\cup \{q\}} \not\in \HT(P)$.\\
If $\tuple{X,Y} \not\in \HT(P)$ then
\begin{itemize}
    \item $\exists r_1\in R_1:\tuple{X,Y}\not\models r_1$, or
    \item $\exists r_4\in R_4:\tuple{X,Y}\not\models r_4$.
\end{itemize}
If $\tuple{X,Y\cup \{q\}} \not\in \HT(P)$ then
\begin{itemize}
    \item $\exists r_2\in R_2:\tuple{X,Y\cup \{q\}}\not\models r_2$, or
    \item $\exists r_0\in R_0:\tuple{X\cup \{q\},Y\cup \{q\}}\not\models r_0\wedge\exists r_3'\in R_3:\tuple{X,Y\cup \{q\}}\not\models r_3'$
\end{itemize}
From the premises of the theorem follows that
\begin{itemize}
    \item $\exists r_3\in R_3:\tuple{Y,Y\cup \{q\}}\not\models r_3$, and
    \item $\exists D\in \asdual{q}{R_0\cup R_2}:\tuple{Y,Y}\models D$, and
    \item $\exists D\in \asdual{q}{R_1\cup R_4}:\tuple{Y,Y}\models D$.
\end{itemize}
In the case that $\tuple{X,Y} \not\in \HT(P)$ derivation rule 6 either produces\\
$\headwoq{r_1}\leftarrow \bodywoq{r_1}\cup \{not\ not\ h(r_1)\}\cup D\cup \nf\nf(\bodywoq{r_3})\cup \nf(\headwoq{r_3})$ or
$\headwoq{r_4}\leftarrow \bodywoq{r_4}\cup \{not\ not\ h(r_4)\}\cup D\cup \nf\nf(\bodywoq{r_3})\cup \nf(\headwoq{r_3})$
which contradicts $\tuple{X,Y}$.
In the case that $\exists r_2\in R_2:\tuple{X,Y\cup \{q\}}\not\models r_2$ derivation rule 3 produces\\
$\head{r_2} \leftarrow \bodywoq{r_2}\cup \{not\ not\ h(r_2))\}\cup D\cup \nf\nf(\bodywoq{r_3})\cup \nf(\headwoq{r_3})$
which contradicts $\tuple{X,Y}$. 

Then, if $\exists r_0\in R_0:\tuple{X\cup \{q\},Y\cup \{q\}}\not\models r_0\wedge\exists r_3'\in R_3:\tuple{X,Y\cup \{q\}}\not\models r_3'$, either $r_3'$ contradicts $\tuple{Y,Y\cup \{q\}}$ then
$\head{r_0} \leftarrow \bodywoq{r_0}\cup \{not\ not\ h(r_0))\}\cup D\cup \bodywoq{r_3'}\cup \nf(\headwoq{r_3'})$ which is constructed by 3 contradicts $\tuple{X,Y}$. If $r_3'$ does not contradict $\tuple{Y,Y\cup \{q\}}$ there is another $r_3\in R_3$ with $\tuple{Y,Y\cup \{q\}}\not\models r_3$ which is also used in the construction of $H(r_0)\cup H^{\setminus q}(r_3')\leftarrow B^{\setminus q}(r_0)\cup\{not\ not\ h(r_0)\}\cup B^{\setminus q}(r_3')\cup \nf\nf(B^{\setminus q}(r_3))\cup \nf(H^{\setminus q}(r_3))\cup D$. This rule contradicts $\tuple{X,Y}$.
Therefore in any case $\tuple{X,Y}\not\in\fSP{P}{q}$.\\
Second, we proof that $\HT_Y(\fSP{P}{q})\supseteq \HT_Y(\f{P}{\{q\}})$. So we suppose that $\tuple{X,Y}\in \HT_Y(\f{P}{\{q\}})$, then $\tuple{X,Y}\in \HT(P)$ and $\tuple{X,Y\cup \{q\}}\in \HT(P)$.\\
From $\tuple{X,Y}\in \HT(P)$ it follows that $\forall r'\in R_1\cup R_4:\tuple{X,Y}\models\headwoq{r'}\vee\tuple{X,Y}\not\models\bodywoq{r'}$, therefore $\tuple{X,Y}$ satisfies all rules derived by 1, 2, 4, 5, 6.\\
If $\tuple{X\cup \{q\},Y\cup \{q\}}\models P$ then $\forall r\in R_0\cup R_2: \tuple{X,Y}\models\head{r}\vee\tuple{X,Y}\not\models\bodywoq{r}$. Then $\tuple{X,Y}$ satisfies all rules derived by 1, 2, 3, 7.\\
If $\tuple{X\setminus \{q\},Y\cup \{q\}}\models P$ then $\forall r_3\in R_3: \tuple{X,Y}\models\head{r_3}\vee\tuple{X,Y}\not\models\bodywoq{r_3}$. Then $\tuple{X,Y}$ satisfies all rules derived by 2, 3, 5, 6, 7.\\
In either case $\tuple{X,Y}\models \fSP{P}{q}$, therefore $\HT_Y(\fSP{P}{q})\supseteq \HT_Y(\f{P}{\{q\}})$, and in the end $\HT_Y(\fSP{P}{q})=\HT_Y(\f{P}{\{q\}})$.
\end{proof}

\begin{lemma}\label{lem_case_6}
Let $P$ be a program over $\Sigma$, $q\in \Sigma$, $Y\subseteq \sign\setminus \{q\}$ and $\tuple{Y,Y} \in \HT(P)$ and $\tuple{Y,Y\cup \{q\}} \in \HT(P)$ and $\tuple{Y\cup \{q\},Y\cup \{q\}} \in \HT(P)$.\\
Then $\HT_Y(\fSP{P}{q})=\HT_Y(\f{P}{\{q\}})$ for a representative $\fgt\in\FSPOP$.
\end{lemma}
\begin{proof}
In this setting, just $\emptyset$ is relevant, i.e. $Rel^Y_{(P,\{q\})} = \{\emptyset\}$, and therefore\\
$\HT_Y(\f{P}{\{q\}})=\HT_Y(P)$.\\
First, we proof that $\HT_Y(\fSP{P}{q})\subseteq \HT_Y(\f{P}{\{q\}})$. So we suppose that $\tuple{X,Y}\not\in \HT_Y(\f{P}{\{q\}})$, then $\tuple{X,Y}\not\in \HT(P)$.\\
From $\tuple{X,Y}\not\in \HT(P)$ it follows that
\begin{itemize}
    \item $\exists r_1\in R_1: \tuple{X,Y}\not\models r_1$, or
    \item $\exists r_4\in R_4: \tuple{X,Y}\not\models r_4$
\end{itemize}
From the premise of the lemma it follows that $\exists D\in\asdual{q}{R_3\cup R_4}:\tuple{X,Y}\models D$. Therefore, in either case the rule that is derived by 4 contradicts $\tuple{X,Y}$.\\
Second, we proof that $\HT_Y(\fSP{P}{q})\supseteq \HT_Y(\f{P}{\{q\}})$. So we suppose that $\tuple{X,Y}\in \HT_Y(\f{P}{\{q\}})$, then $\tuple{X,Y}\in \HT(P)$.\\
From $\tuple{X,Y}\in \HT(P)$ it follows that $\forall r'\in R_1\cup R_4:\tuple{X,Y}\models\headwoq{r'}\vee\tuple{X,Y}\not\models\bodywoq{r'}$, therefore $\tuple{X,Y}$ satisfies all rules derived by 1, 2, 4, 5, 6.\\
There are two cases, either $\tuple{X\cup \{q\},Y\cup \{q\}}\models P$, then $\forall r\in R_0\cup R_2: \tuple{X,Y}\models\head{r}\vee\tuple{X,Y}\not\models\bodywoq{r}$. Then $\tuple{X,Y}$ satisfies all rules derived by 1, 2, 3, 7.\\
Or $\tuple{X\cup \{q\},Y\cup \{q\}}\not\models P$, then $\exists r\in R_0\cup R_2:\tuple{Y,Y}\not\models r$ and therefore $\forall D\in \asdual{q}{R_0\cup R_2}:\tuple{Y,Y}\not\models D$, which means that $\tuple{X,Y}$ satisfies all rules derived by 3 and 7.\\
In either case $\tuple{X,Y}\models \fSP{P}{q}$, therefore $\HT_Y(\fSP{P}{q})\supseteq \HT_Y(\f{P}{\{q\}})$, and in the end $\HT_Y(\fSP{P}{q})=\HT_Y(\f{P}{\{q\}})$.
\end{proof}

\section{Detailed Comparison}
In this appendix, we provide more detailed information on the comparison for each of the three existing approaches on syntactic forgetting in answer set programming.
We clearly identify their short-comings and limitations in comparison to our novel operator.

To make this material self-contained, in each case, we first recall their concise definition, as presented in \cite{GoncalvesKL16}, and then discuss the crucial differences.
\subsection{Strong and Weak Forgetting}
Zhang and Foo \shortcite{ZhangF06} introduced two syntactic operators for normal logic programs, termed Strong and Weak Forgetting.
Both start with computing a reduction corresponding to the well-known weak partial evaluation (WGPPE) \cite{BrassD99}, defined as follows: for a normal logic program $P$ and $q\in\sign$, $R(P,q)$ is the set of all rules in $P$ and all rules of the form $\rhead{r_1}\la \rbody{r_1}\setminus\{q\}\cup \rbody{r_2}$ for each $r_1,r_2\in P$ s.t.\ $q\in\rbody{r_1}$ and $\rhead{r_2}=q$.
Then, the two operators differ on how they subsequently remove rules containing $q$, the atom to be forgotten.
In Strong Forgetting, all rules containing $q$ are simply removed:
\[
\fgt_{\strong}(P,q) = \{r\in R(P,q)\mid q\not\in \sign(r)\}
\] 
In Weak Forgetting, rules with occurrences of $\nf q$ in the body are kept, after $\nf q$ is removed.
\begin{align*}
\fgt_{\weak}(P,q) & = \{\rhead{r}\la\rbody{r}\setminus\{\nf q\}\mid \\
&  \;\;\;\;\;\; r\in R(P,q), q\not\in \rhead{r}\cup\rbody{r}\}
\end{align*}
The motivation for this difference is whether such $\nf q$ is seen as support for the rule head (Strong) or not (Weak). 
Both operators are closed for normal programs.

It is easy to identify the one communality of these two syntactic operators and the new operator $\fgt_{SP}$. 
If we restrict step {\bf 1a} in Def.~\ref{def:OurForgetting} to normal logic programs, then it does correspond to the calculation of the new rules obtained in $R(P,q)$.

The operator $\fgt_{SP}$ is, of course, applied to a larger class of programs, which is important and necessary as forgetting while satisfying \pSP\ cannot be closed for normal programs (cf.~\cite{KnorrA14}), but there are further crucial differences.
First, the normalization applied in \cite{ZhangF06} is limited to the elimination of the two kinds of tautological rules (cf.~Def.~2 in \cite{ZhangF06}) possible for normal programs (out of the three cases we consider).
Our normalization in Def.~\ref{def:normalForm} includes steps 3.\ and 4.\ which are applicable to normal programs, but not considered in \cite{ZhangF06}.
More importantly, both operators would not do any actual transformations/replacements w.r.t.\ negated atoms in the body.
Rather, either the rules are simply deleted (Strong Forgetting), or only the negated atoms are deleted.
This, yields counter-intuitive results of forgetting for which \pSP\ is not satisfied even though forgetting while satisfying \pSP\ would be possible.

\begin{example}\label{ex:showcase}
Consider forgetting about $q$ from $P=\{p\la \nf q;$ $q\la \nf c\}$.    
In both cases, $R(P,q)=P$, since (WGPPE) is not applicable.
Then, in Strong Forgetting all rules are deleted, i.e., $\fgt_{\strong}(P,q)=\emptyset$, whereas in Weak Forgetting, in the first rule, $\nf q$ is deleted only and the second rule entirely, i.e., $\fgt_{\weak}(P,q)=\{p\la\}$.
The desired result, that, e.g., would allow us to infer $p$ if a fact $c\la$ is added to the program, would contain (only) the rule $p\la \nf\nf c$ instead.
Hence, in both cases, undesired results are obtained.
\end{example}
In summary, while both operators behave in a desired manner when forgetting only involves positive atoms in the rule bodies, in the presence of a negated atom (to be forgotten) in some body, they cease to be of interest, as strong and weak forgetting resume to essentially strong and weak deletion with undesired effects on forgetting results, both, arguably from an intuitive point of view, as well as due to the fact that \pSP\ is trivially not satisfied.

\subsection{Semantic Forgetting}

Eiter and Wang \shortcite{EiterW08} proposed Semantic Forgetting to improve on some of the shortcomings of the two purely syntax-based operators $\fgt_{\strong}$ and $\fgt_{\weak}$.
Semantic Forgetting introduces a class of operators for consistent disjunctive programs\footnote{Actually, classical negation can occur in scope of $\nf$, but due to the restriction to consistent programs, this difference is of no effect \cite{GL91}, so we ignore it here.} defined as follows
\[
\classF_{\sem} = \{\fgt\mid \as{\f{P}{V}} = \Min{\as{P}_{\parallel V}}\}
\]
where $\Min{S}$ refers to the minimal elements in $S$.
The basic idea is to characterize a result of forgetting just by its answer sets, obtained by considering only the minimal sets among the answer sets of $P$ ignoring $V$. 
Three concrete algorithms are presented, two based on semantic considerations and one syntactic.
Unlike the former, the latter is not closed for normal and disjunctive programs, since double negation is required in general \cite{EiterW08}.

\begin{figure}[t!]
\includegraphics[width=0.87\textwidth]{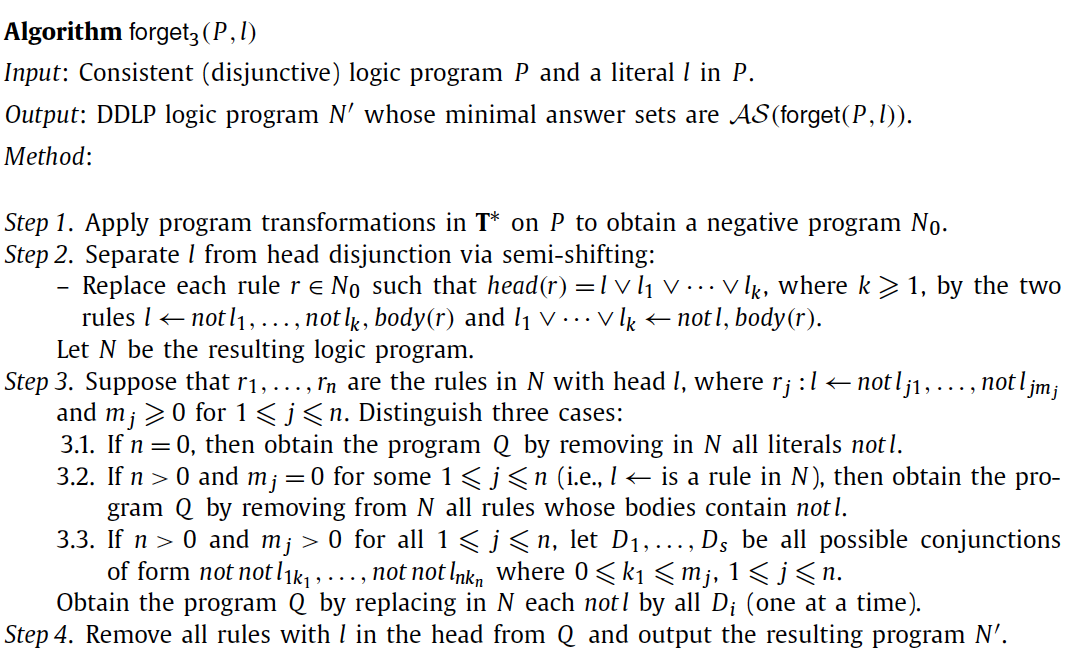}
\centering
\caption{Syntactic algorithm $\forget_3$}
\label{fig:forget3}
\end{figure}

We now focus on the detailed comparison between this syntactic operator, called $\forget_3$, and $\fgt_{SP}$.
To ease the comparison, we recall it in Fig.~\ref{fig:forget3}.

To begin with, $\forget_3$ requires a consistent disjunctive program whereas $\fgt_{SP}$ is more general and can be applied to programs without answer sets and, of course, to the general case of extended programs.\footnote{Actually, forgetting a single literal is considered for $\forget_3$, but again, since for consistent programs, we can always transform the program with literals (including classical negation) into one without, this difference is not of importance here.}
In terms of output, $\forget_3$ produces an extended program (called DDLP in \cite{EiterW08}) whose minimal answer sets coincide with $\as{\f{P}{l}}$ for all $\fgt\in \classF_{\sem}$. 
The operator $\fgt_{SP}$ returns an extended program, that corresponds to the semantic definition of class $\FSPOP$, i.e., if $\Omega$ is not satisfied, i.e., \pSP\ holds for this forgetting instance, then (without any minimization) the answer sets of the result correspond precisely to those of the input program (modulo the forgotten atoms). Otherwise, if $\Omega$ is satisfied, then the answer sets of the result are a superset of those of the input program (again modulo the forgotten atoms).

Now, although at first glance, Step 1 of $\forget_3$ may seem very similar to the normalization applied here for $\fSPOP$, there are crucial differences.
In more detail, Step 1 considers a collection of program transformations, namely, \emph{Elimination of Tautologies}, \emph{Elimination of Head Redundancy}, \emph{Positive Reduction}, \emph{Negative Reduction}, \emph{Elimination of Implications}, \emph{Elimination of Contradictions}, and \emph{Unfolding} (please refer to 4.3.1 in \cite{EiterW08} for their concise definitions).
Among these seven, four transformations have a correspondence in the construction of the normal form, since Elimination of Tautologies and Elimination of Contractions correspond to 1.\ of Def.~\ref{def:constructionNormalForm}, Elimination of Head Redundancy corresponds to 3.\ of Def.~\ref{def:constructionNormalForm}, and Elimination of Implications to 4.\ of Def.~\ref{def:constructionNormalForm}.
Note that there is no correspondence to 2.\ of Def.~\ref{def:constructionNormalForm} as $\forget_3$ does not consider double negation.
The other three transformations are not part of the normal form considered here, and, as such, to a large extent responsible for the inadmissibility of $\forget_3$ as a syntactic operator for satisfying \pSP. 

Positive Reduction essentially removes negated atoms $\nf c$ from the body of a rule if there is no rule in the given program whose head contains $c$.
This is a reasonable step in the context of Semantic Forgetting, where one is only interested in the answer sets of the current program.
If we want to take into account the answer sets of the program together with other rules over the remaining language, then this clearly is not suitable.

\begin{example}
Recall the program from Ex.~\ref{ex:showcase} where we forget about $q$ from $P=\{p\la \nf q;$ $q\la \nf c\}$. By Positive Reduction, $\nf c$ is eliminated in Step 1, leading eventually to the forgetting result $\emptyset$, which, as argued in Ex.~\ref{ex:showcase} is not desired. Rather a single rule $\{p\la\nf\nf c\}$ would be expected.
\end{example}
This is also the reason why the forgetting result for the example program $P'$ in Sec.~\ref{sec:relWork} returns $\{p\la\}$ and not $\{p\la \nf c\}$.
In fact, positive reduction alone already suffices to not use $\forget_3$, but there are further problems.

At first glance, Unfolding seems to coincide with the idea of WGPPE in \cite{ZhangF06} and step {\bf 1a} in Def.~\ref{def:OurForgetting}.
The crucial difference is that the latter only apply this to the atom to be forgotten, whereas Unfolding is applied to all occurrences of positive atoms.
This, again, is suitable for the approach of Semantic Forgetting, as it does not jeopardise the answer sets of the program itself, and it subsequently allows for a simpler definition of forgetting (as no positive atoms can occur in the rule bodies of a program in negative program).
However, this step causes several problems including the satisfaction of \pSP.

\begin{example}\label{ex:HornProgramDifference}
Consider forgetting about $q$ from the program $P=\{a\la b;$ $b\la a;$ $q\la\}$.
This program says that $q$ is true, and that both $a$ and $b$ are false, but if we manage to derive one of them, then also the other has to be true.

Now, by Unfolding, we can replace $a\la b$ with $a\la a$. 
This rule is a tautology, so it can be removed.
Finally, rule $b\la a$ is removed also by unfolding (as a special case of unfolding as argued likewise in \cite{EiterW08} for the running example in Sec.~4.3.1 justifying the removal of $r'_5$).
The result is indeed a negative program, but forgetting about q yields the empty program, and the dependency between $a$ and $b$ (if one is true, then the other also has to be true) has been lost in the course of forgetting using $\forget_3$.
\end{example}
Hence, general Unfolding is unsuitable when forgetting while satisfying \pSP.
In addition, it considerably reduces the similarity between the input program and its forgetting result, unnecessarily increasing the number of rules in the resulting program, affecting also the rules that do not contain the atom(s) to be forgotten.

Negative Reduction, which basically deletes rules whenever they contain a negated body atom for which the program contains a fact, further increases the difference between the original program and its output, again effecting atoms the rules that do not contain the atom(s) to be forgotten.

Step 2 of $\forget_3$ causes further differences: its objective is to essentially separate occurrences of atoms in disjunctions so that subsequently the atom to be forgotten appears alone in the rule head.
This transformation does not preserve strong equivalence and is not part of the definition of $\fSPOP$.

The construction in Step 3 is a simplified version of the construction of $\asdual{q}{P}$ and case {\bf 4} of Def.~\ref{def:OurForgetting}, simplified due to the fact that the input not allows double negation.
Finally, Step 4, corresponds to removing/not including any rules containing the atom to be forgotten.

In summary, there are commonalities between $\forget_3$ and $\fSPOP$, but, even restricted to Horn programs (as in Ex.~\ref{ex:HornProgramDifference}), there are crucial differences which make the usage of $\forget_3$ inadmissible, since we want that \pSP\ is satisfied whenever this is possible. 
The essential reason for that is that several of the syntactic transformations applied in $\forget_3$ are justified by the equivalence of answer sets only, which is not sufficient for \pSP. 
We want to stress again that $\forget_3$ cannot be iterated, since, in general, it produces extended programs (with double negation).
The operator $\fSPOP$ overcomes this problem, but, due to this, requires a more involved definition that handles double negation.

\subsection{Strong AS-Forgetting}
Knorr and Alferes \shortcite{KnorrA14} introduced Strong AS-Forgetting with the aim of preserving not only the answer sets of $P$ itself but also those of $P\cup R$ for any $R$ over the signature without the atoms to be forgotten, in the sense of \pSP.
The notion is defined abstractly for classes of programs $\mathcal{C}$.
\begin{align*}
\classF_{\Sas} & = \{\fgt\mid \as{\f{P}{V}\cup R}=\as{P\cup R}_{\parallel V} \text{ for all}\\
& \;\;\;\;\; \text{ programs } R\in \mathcal{C} \text{ with } \sign(R)\subseteq \sign(P)\setminus V\}
\end{align*}
A concrete operator is defined for a non-standard class of programs (with double negation and without disjunction), but not even closed for normal programs.

While Strong AS-Forgetting provides the expected results in certain cases, its applicability is severely limited to certain programs within a non-standard class with double negation, but without disjunction. Additionally, its result often does not belong to the class of programs for which it is defined, preventing its iterative use. $\fSPOP$ overcomes all the shortcomings of strong as-forgetting and satisfies \pSP\ whenever this is possible.

\end{document}